%% file: neurips_2026.tex
\documentclass{article}

    \PassOptionsToPackage{numbers, compress}{natbib}

 \usepackage[main, final]{neurips_2026}

\usepackage[utf8]{inputenc} 
\usepackage[T1]{fontenc}    
\usepackage{hyperref}       
\usepackage{url}            
\usepackage{booktabs}       
\usepackage{amsfonts}       
\usepackage{nicefrac}       
\usepackage{microtype}      
\usepackage{xcolor}         

\usepackage{amsmath,amsthm,amssymb,mathrsfs}
\usepackage{tikz}
\usetikzlibrary{positioning,shapes.geometric,arrows.meta,fit}
\usepackage{enumitem}
\usepackage{subcaption}

\title{Learning Sparse Compositional Functions with Norm-Constrained Neural Networks}

\author{%
  Shuo Huang \\
   Istituto Italiano di Tecnologia \\
  \texttt{shuo.huang@iit.it} \\
  \And
  Lorenzo Fiorito \\
  University of Genova \\
  \texttt{lorenzo.fiorito@edu.unige.it} \\
  \And
  Lorenzo Rosasco \\
University of Genova, Italy \\
  Istituto Italiano di Tecnologia, Italy \\
  \texttt{Lorenzo.Rosasco@unige.it} \\
  \And
  Tomaso Poggio \\
  Massachusetts Institute of Technology \\
  \texttt{tp@csail.mit.edu} \\
}

\input{macros}

\begin{document}

\maketitle

\begin{abstract}
The ability of deep neural networks to learn hierarchical features is widely regarded as a key mechanism underlying their success in high-dimensional learning. Existing theory partially supports this view by establishing approximation rates based on parameter counts and sample complexity guarantees for compositional models without incurring the curse of dimensionality (CoD). To study overparameterized regimes, where the number of parameters exceeds the sample size, we develop a framework that measures complexity via the parameter norm. Within this approach, we establish approximation rates and excess risk bounds for learning sparse compositional functions whose compositional structure is represented by directed acyclic graphs (DAGs), using Frobenius norm-constrained deep neural networks.
Our results have broad applicability since every function that is efficiently Turing computable admits sparse compositional representations. In particular, we cover a range of representative models, including multi-index models, binary tree structures, and general compositional architectures. The rates we derive show that deep networks can exploit the compositional structure of the target functions, effectively avoiding the CoD through hierarchical representations.
\end{abstract}

\section{Introduction}
The success of deep learning in high-dimensional tasks such as vision and language has motivated the search for principles underlying its effectiveness. 
A central insight is that depth enables efficient hierarchical representations, progressively extracting features of increasing complexity \citep{tishby2015deep,cagnetta2024deep,ren2026provable}. 
This organization parallels hierarchical processing in the visual cortex \citep{van1983hierarchical,felleman1991distributed,grill2004human,dicarlo2012does} and arises in structured data such as syntactic parse trees and object–part decompositions \citep{chomsky1956three,biederman1987recognition}. 
These observations have led to hierarchical compositional data models that play a key role in theoretical analyses of deep learning efficiency \citep{allen2019can,abbe2022merged,sclocchi2025phase}. 
A prominent example is the class of sparse compositional functions, where each component depends on only a few variables \citep{mhaskar2017and,poggio2017and,poggio2024compositional}.
This perspective is consistent with architectures incorporating locality, such as convolutional neural networks and transformers \citep{krizhevsky2012imagenet,vaswani2017attention}.  Furthermore, the class of functions that are compositionally sparse is broad, since functions that are efficiently computable by a Turing machine admit compositionally sparse representations \citep{poggio2025efficiently,danhofer2025position}. From a statistical learning perspective, the strength of this principle lies in its capacity to reduce sample complexity, that is the number of samples required to achieve a target accuracy,  therefore improving learning efficiency \citep{vapnik2013nature}.

Classical results establish that sample complexity scales exponentially with the input dimension, the so-called curse of dimensionality (CoD) \citep{donoho2000high}. 
For compositional models, however, deep neural network results show that the required sample size depends on the dimensions of the component functions rather than the ambient dimension, thereby alleviating the  CoD \citep{poggio2017and,schmidt2020nonparametric,kohler2021rate,nagler2026optimal}. 
These guarantees, however,  are expressed in terms of the network parameter count, which is required to be smaller than the sample size. 
Thus these results do not extend to modern overparameterized regimes, where number of  parameters  vastly exceed the number of training samples, as in large foundation models including GPT-4 \citep{achiam2023gpt}, Gemini \citep{team2023gemini}, and LLaMA \citep{touvron2023llama}. 
This discrepancy raises two fundamental questions: what complexity measure can appropriately characterize large-scale, overparameterized neural networks beyond parameter count? And how can one establish corresponding sample complexity guarantees for hierarchical learning in this regime?


Recent advances in optimization theory provide insight into the first question. 
In overparameterized regimes, empirical risk minimization admits infinitely many interpolating solutions, yet gradient-based methods exhibit an implicit bias toward low-complexity ones characterized by certain norms \citep{soudry2018implicit,hastie2022surprises, neyshabur2015norm,gunasekar2018implicit}. 
These findings indicate that effective complexity is governed by some norm structures across layers rather than raw parameter counts.


Based on these observations, we investigate the sparse compositionality principle by analyzing the sample complexity of learning sparse compositional functions with norm-constrained neural networks. Our main contributions are summarized as follows.
\begin{itemize}
    \item  We establish approximation rates for H\"older spaces characterized by Frobenius norm constraints, rather than by the parameter count typically used in classical approximation theory. This shifts the complexity measure from network size to norm-controlled capacity, thereby extending the analysis to the overparameterized regime.
    \item Building on the DAG-based formulation of sparse compositional functions, we further provide the first approximation rates for such models in terms of the norm of parameters of neural networks. Explicit rates for multi-index models, binary trees, and more general architectures demonstrate the mitigation of the curse of dimensionality.
    \item Moreover, we show that the sample complexity of learning sparse compositional functions depends exponentially on the intrinsic input dimension of the local constituent functions, rather than on the ambient dimension. This establishes that hierarchical locality mitigates the curse of dimensionality even when networks are highly overparameterized, but controlled through norm constraints.
\end{itemize}

The remainder of the paper is organized as follows. 
Section~\ref{sec:related work} reviews related work. 
In Section~\ref{sec:apx_single}, we establish norm-based approximation rates for both single H\"older spaces and sparse compositional H\"older spaces.
Section~\ref{sec:statistical} presents the statistical analysis and derives sample complexity bounds, with Subsection~\ref{sec:discussion} discussing broader perspectives and outlining potential research directions.
Finally, Section~\ref{sec:conclusion} concludes the paper.

\section{Related work}\label{sec:related work}

\paragraph{Compositionality principle}
Hierarchical compositional structure has been supported from multiple theoretical perspectives as a fundamental principle underlying the effectiveness of deep neural networks. 
In approximation theory, deep compositional networks, convolutional neural networks, and transformers have been shown to adapt to compositional targets, achieving rates governed by the intrinsic dimensionality of intermediate functions rather than the ambient input dimension \citep{mhaskar2017and,mao2023approximating,dahmen2025compositional,shi2025approximation}.  Tree tensor networks provide another hierarchical representation of high-dimensional functions. \citet{bachmayr2021approximation} establish approximation results for Sobolev and compositional classes, and \citet{ali2023approximation} study approximation classes of tensorized multivariate functions. Their bounds measure complexity through tensor ranks and representation size, rather than the product Frobenius norm of network parameters considered here
A similar phenomenon appears in statistical learning theory. Sample complexity bounds for compositional models scale with the dimensions of constituent functions, thereby mitigating the curse of dimensionality \citep{schmidt2020nonparametric,kohler2021rate,huang2025learning}. 
Beyond those, this structure also yields computational advantages. Recent works show that gradient-based methods can efficiently learn structured models such as single- and multi-index models as well as staircase-type functions, with complexity characterized through quantities such as the {information exponent} and the {leap exponent} \citep{bietti2023learning,abbe2022merged,dandi2025computational,zhang2025neural,defilippis2026optimal}. Furthermore, \citet{poggio2024compositional,danhofer2025position} claim that sparse compositionality is implied by efficient Turing computability.

\paragraph{Complexity measures for neural networks}

The complexity of neural networks has been characterized through various measures aimed at controlling the generalization error, i.e., the discrepancy between population and empirical risk. 
Classical analyses primarily relied on parameter counting arguments and VC-dimension bounds \citep{baum1988size,anthony2009neural,yang2023nearly}, which are often inadequate for modern deep networks \citep{zhang2016understanding,rauchwerger2026dense}. 
Subsequent work introduced norm-, margin-, and sharpness-based measures to better capture capacity, see, e.g., \citep{neyshabur2015norm,arora2018stronger,keskar2016large}. 



\paragraph{Theoretical analysis of overparameterization}
The theoretical analysis of overparameterized neural networks has been extensively studied from optimization and generalization perspectives.  In the interpolation regime, gradient-based methods exhibit an implicit bias toward low-complexity solutions rather than arbitrary interpolants \citep{soudry2018implicit,gunasekar2018implicit,ji2020directional}. This phenomenon is closely related to benign overfitting and double descent behavior observed in modern deep models \citep{belkin2019reconciling,bartlett2020benign,hastie2022surprises}. Another influential line of work analyzes the infinite-width limit via the neural tangent kernel (NTK) and mean-field perspectives, characterizing training dynamics in highly overparameterized regimes \citep{jacot2018neural,chizat2019lazy,mei2019mean}. However, comparatively less is understood from an approximation and statistical efficiency viewpoints in overparameterized regimes.


\section{Approximation analysis of Frobenius norm-constrained NNs}\label{sec:apx_single}

In this section, we introduce the class of neural networks under consideration and establish their approximation properties, including explicit rates for approximating the  H\"older spaces and sparse compositional H\"older spaces.

\subsection{Neural networks with Frobenius norm constraint}\label{subsec:intro_NN_Fnorm}
Let \(D \in \mathbb{N}\) denote the depth of the neural network, and let
\(\{N_\ell\}_{\ell=0}^{D+1}\) denote the layer widths, where in particular
$N_0 = d$ corresponds to the input dimension and $N_{D+1}=k$ to the output dimension.
Consider functions $\phi : \R^d \to \R^k$ that can be parameterized by a deep ReLU neural network of the form
\begin{equation}\label{equ:relu_network}
   \phi(x) = A_D \sigma\!\big( A_{D-1} \sigma( \cdots \sigma( A_0 x + b_0 ) \cdots ) + b_{D-1} \big). 
\end{equation}
Here, the input $x\in \R^d$, linear matrix $A_\ell\in \R^{N_{\ell +1}\times N_{\ell}}$, bias $b_\ell \in \R^{N_{\ell+1}}$, and $\sigma(x) = \max\{0,x\}$ is the ReLU activation function applied componentwise.
Denote $\theta = ({A}_\ell, b_\ell)_{\ell=0}^D$ and let $f_\theta$ be the corresponding realized function.
The function class generated by neural networks of the form \eqref{equ:relu_network} with depth $D$ and maximal width $W := \max_{\ell} N_\ell$ is denoted by $\CN\CN(W,D)$. 
The Frobenius norm of a matrix $A = (a_{i,j}) \in \mathbb{R}^{m \times n}$ is given by
\(
\|A\|_F := \left( \sum_{i=1}^m \sum_{j=1}^n |a_{i,j}|^2 \right)^{1/2}.
\)
We equip the class $\CN\CN(W,D)$ with the multiplicative Frobenius norm constraint
\begin{equation}\label{equ:kappa_theta}
\kappa(\theta)
:= \|A_D\|_{F}\prod_{\ell=0}^{D-1} \|\tilde{A}_\ell\|_F=  \|A_D\|_{F} \prod_{\ell=0}^{D-1} 
\sqrt{ \|A_\ell\|_{F}^{2} + \|b_\ell\|_{2}^{2}+1 }, \quad \widetilde A_\ell=
\begin{pmatrix}
A_\ell & b_\ell\\
0 & 1
\end{pmatrix}.
\end{equation}
The corresponding norm-constrained function class is
\begin{equation}\label{equ:class_bias}
  \mathcal{N}\mathcal{N}(W,D,K)
:= \left\{ f_\theta \in \CN\CN(W,D) \;:\; \kappa(\theta) \le K \right\}.  
\end{equation}

The term $+1$ in $\kappa(\theta)$ in \eqref{equ:kappa_theta} is included to handle
bias terms in \eqref{equ:relu_network} through the standard homogeneous-coordinate representation.  Indeed,
by appending a constant coordinate as given in $\widetilde{A}_\ell$, a 
network with bias in \eqref{equ:relu_network} can be written as a bias-free augmented network on an input $(x^\top,1)^\top$.
Hence $\kappa(\theta)$ is the Frobenius product norm of the augmented network. This definition allows us to handle biases consistently
within norm-based capacity estimates for networks without bias \citep{neyshabur2015norm,bartlett2017spectrally,golowich2020size}; See Appendix~\ref{sec:pf_stats} for more details.
We impose a multiplicative norm constraint 
rather than an additive one. This choice reflects the compositional structure of deep networks: forward propagation and the resulting Lipschitz constant scale multiplicatively across layers.
Moreover, ReLU networks without bias are positively homogeneous and invariant under layer-wise rescaling, a symmetry preserved by product norms but not by additive constraints.
Generalization bounds based on Rademacher or margin analysis
naturally involve such products
\citep{neyshabur2015norm,bartlett2017spectrally,golowich2020size}.
We adopt the Frobenius norm at each layer, which measures the overall magnitude of parameters. Compared to spectral or path norms, it provides a simple and widely used measure corresponding to $\ell_2$ regularization (often corresponding to weight decay) in practice \citep{krogh1991simple,zhang2018three}.

\subsection{Approximation rates for H\"older spaces}\label{subsec:apx_s_H}

Let $\alpha=r+\beta$ with $r\in\mathbb{N}_0$ and $\beta\in(0,1]$. 
The {unit ball} of H\"older space $\mathcal{C}^\alpha([-1,1]^d)$ of order $\alpha$ consists of all functions 
$h:[-1,1]^d\to\mathbb{R}$ that are $r$ times continuously differentiable and satisfy
\[
\|h\|_{\mathcal{C}^\alpha}
:=
\sum_{|\gamma|\le r}
\|D^\gamma h\|_{L^\infty([0,1]^d)}
+
\max_{|\gamma|=r}
\sup_{x\neq y}
\frac{|D^\gamma h(x)-D^\gamma h(y)|}{\|x-y\|_\infty^\beta} \le {1}.
\]
Here $\gamma=(\gamma_1,\ldots,\gamma_d)\in\mathbb{N}_0^d$ is a multi-index with 
$|\gamma|=\sum_{i=1}^d \gamma_i$, and 
$D^\gamma h=\partial_1^{\gamma_1}\cdots\partial_d^{\gamma_d} h$. This unit-ball assumption is adopted for convenience and can be extended directly to any bounded ball.

The following theorem establishes the approximation rate of neural networks with Frobenius norm constraints for H\"older spaces, whose proof is presented in Appendix~\ref{subsec:pf_apx_s_H}. We write $A \lesssim B$ (resp.\ $A \gtrsim B$) if $A \le cB$ (resp.\ $A \ge cB$) for some $c>0$, and $A \asymp B$ if both hold.


\begin{theorem}[Approximation of H\"older spaces]\label{thm:apx_s_F}
For any $K\ge1$ and $h \in \mathcal{C}^\alpha([0,1]^{d})$ with $\alpha = r + \beta$, 
where $r \in \mathbb{N}_0$ and $\beta \in (0,1]$, 
there exists a neural network 
$\phi \in \mathcal{N}\mathcal{N}(W, D, K)$ 
with depth $D = 2\left\lceil \log_2({d}+r) \right\rceil+2$ 
and width $W \gtrsim K^{\frac{2{d}+\alpha}{{d}(D+1)+2}}$ 
such that
\[
\|h - \phi\|_{L^\infty([0,1]^d)} 
\;\lesssim\; K^{-\frac{2\alpha}{2+{d}\left(D+1\right)}}.
\]
\end{theorem}

The theorem provides a parameter norm-based analogue of classical approximation rates. The resulting rate
is polynomial in $K$, improves with the H\"older smoothness $\alpha$, and
deteriorates with the input dimension $d$. Although the exponent also depends
on the depth $D$, it is relatively small and grows logarithmically with $d$.
It holds for neural networks with any width $W$ exceeds a polynomial threshold in $K$, thereby covering the model in the overparameterized regime.  In other words, once the width exceeds the
stated polynomial threshold in $K$, the approximation rate is governed by the
norm budget rather than by the width.
A key step in the proof is
to use shallow but wide architectures to approximate square functions, which
helps to obtain norm constraints independent of the width. This type of
width-independent norm control is not provided by classical approximation
results formulated in terms of parameter count, such as
\citep{yarotsky2017error,shen2019deep,lu2021deep}.
See the discussion following Lemma~\ref{lem:approx_x2} in Appendix~\ref{subsec:pf_intro_NN_F} for quantitative details. We also refer to \citep{yang2024deeper} for a comparison between wider and deeper architectures from approximation and generalization viewpoints.

The closest norm-based result is given in \citep{jiao2023approximation}, which studies approximation under operator-norm constraints.
In contrast, we work on Frobenius norms, which align more naturally with $\ell_2$ regularization in practice.  Both proofs proceed through square functions, multiplication, monomials, and local Taylor polynomials. Our contribution requires new bounds for the accumulation of Frobenius norms in these constructions and extends the analysis to sparse compositional classes on DAGs.
Another related line of work studies approximation rates characterized explicitly by function-space norms in the infinite-width limit, where width no longer affects the rate. 
In this regime, integral representations of shallow neural networks have been well studied when endowed with total variation or Barron-type norms \citep{bach2017breaking,yang2024nonparametric,siegel2024sharp}. 
We highlight this line of work because it raises intriguing questions about the functional-analytic characterization of deep neural networks and its implications for norm-based approximation theory.




\subsection{Approximation rates for sparse compositional H\"older spaces}
In this section, we formalize the hierarchical compositional model as a class of sparse compositional functions using the notation of DAGs, and establish the approximation rates of neural networks under Frobenius norm constraints.
\subsubsection{Sparse compositional function} In hierarchical compositional functions, the dependency structure across levels can be highly heterogeneous: constituent functions may differ in their input dimensionality, interaction patterns, and structural complexity. Such variability makes a uniform layered description inadequate. Inspired by prior work \citep{mhaskar2017and,poggio2017and}, DAGs provide a flexible formalism to encode these non-uniform dependencies and varying local dimensions. 
Let $(\CV,\CE)$ be a DAG with nodes $\CV=\bigcup_{\ell=0}^L \CV^{(\ell)}$, where $\CV^{(0)}=\{x_1,\dots,x_d\}$ and $\CV^{(L)}=\{v^{(L)}_{\mathrm{out}}\}$. Let $|\CV^{(\ell)}|$ denote the number of nodes at level $\ell$. The edge set satisfies $\CE=\bigcup_{\ell=1}^L \CE^{(\ell)}$ with $\CE^{(\ell)}\subseteq \CV^{(\ell-1)}\times\CV^{(\ell)}$. For $v^{(\ell)}\in\CV^{(\ell)}$, define $\mathrm{pa}(v^{(\ell)})=\{u\in\CV^{(\ell-1)}:(u,v^{(\ell)})\in\CE^{(\ell)}\}$. We formalize a sparse compositional function $f: \R^d \to \R$ associated with a
  DAG as
\begin{align}\label{equ:cmp}
  f(x)= g_{\CV^{(L)}}\circ g_{\CV^{(L-1)}} \circ \cdots \circ g_{\CV^{(0)}}(x).   
\end{align}
 With a slight abuse of notation, we write the cardinality of the set $\mathrm{pa}(v^{(\ell)})$ as $d_{\text{in}(v^{(\ell)})}$ to represent the input dimension for the associated local function 
\(
g_{v^{(\ell)}}:\R^{d_{\text{in}(v^{(\ell)})}}\to\R.
\)
Note that the cardinality $|\CE^{(\ell)}| =\sum_i  d_{\text{in}(v_i^{(\ell)})}$.
Hence, 
\(
  g_{\mathcal{V}^{(\ell)}} :
\R^{|\CE^{(\ell)}|} \to \R^{|\CV^{(\ell)}|}\) with \(g_{\CV^{(\ell)}} =(g_{v_1^{(\ell)}}, \ldots, g_{v_{|\CV^{(\ell)|}}^{(\ell)}})
\)
is a collection of functions associated with $\ell$-th level nodes. A simple example of DAG illustration with three levels is given in the left of Figure~\ref{fig:dag_block}.

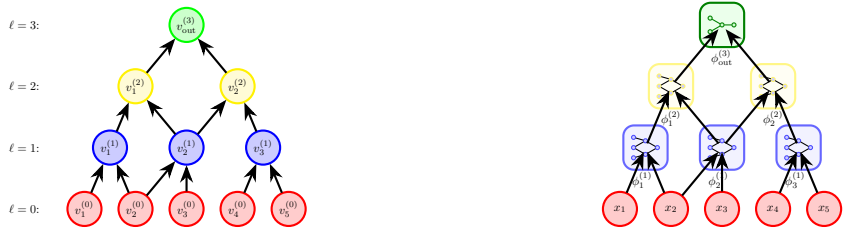
\begin{figure}[ht]
\centering

\begin{subfigure}[t]{0.48\linewidth}
\centering
\begin{tikzpicture}[scale=0.45, every node/.style={transform shape},
    base/.style={circle, draw, thick, minimum size=1.0cm, inner sep=0pt},
    input/.style={base, draw=red, fill=red!20},
    hidden1/.style={base, draw=blue, fill=blue!20},
    hidden2/.style={base, draw=yellow, fill=yellow!20},
    output/.style={base, draw=green, fill=green!20},
    arrow/.style={-Stealth, thick}
]
\node at (-1.2,0) [left] {$\ell=0$:};
\node at (-1.2,1.8) [left] {$\ell=1$:};
\node at (-1.2,3.6) [left] {$\ell=2$:};
\node at (-1.2,5.4) [left] {$\ell=3$:};

\node[input] (v10) at (0,0) {$v_1^{(0)}$};
\node[input] (v20) at (1.5,0) {$v_2^{(0)}$};
\node[input] (v30) at (3.0,0) {$v_3^{(0)}$};
\node[input] (v40) at (4.5,0) {$v_4^{(0)}$};
\node[input] (v50) at (6.0,0) {$v_5^{(0)}$};

\node[hidden1] (v11) at (0.75,1.8) {$v_1^{(1)}$};
\node[hidden1] (v21) at (3.0,1.8) {$v_2^{(1)}$};
\node[hidden1] (v31) at (5.25,1.8) {$v_3^{(1)}$};

\node[hidden2] (v12) at (1.5,3.6) {$v_1^{(2)}$};
\node[hidden2] (v22) at (4.5,3.6) {$v_2^{(2)}$};

\node[output] (vout) at (3.0,5.4) {$v_{\mathrm{out}}^{(3)}$};

\draw[arrow] (v10) -- (v11); \draw[arrow] (v20) -- (v11);
\draw[arrow] (v20) -- (v21); \draw[arrow] (v30) -- (v21);
\draw[arrow] (v40) -- (v31); \draw[arrow] (v50) -- (v31);
\draw[arrow] (v11) -- (v12); \draw[arrow] (v21) -- (v12);
\draw[arrow] (v21) -- (v22); \draw[arrow] (v31) -- (v22);
\draw[arrow] (v12) -- (vout); \draw[arrow] (v22) -- (vout);
\end{tikzpicture}
\end{subfigure}
\hfill
\begin{subfigure}[t]{0.48\linewidth}
\centering
\begin{tikzpicture}[scale=0.45, every node/.style={transform shape},
    input/.style={circle, draw=red, fill=red!20, thick, minimum size=1.0cm, inner sep=0pt},
    arrow/.style={-Stealth, thick},
    block1/.style={draw=blue!60, thick, rounded corners=4pt, fill=blue!5},
    block2/.style={draw=yellow!60, thick, rounded corners=4pt, fill=yellow!5},
    outblock/.style={draw=green!50!black, thick, rounded corners=4pt, fill=green!8},
    neuron1/.style={circle, draw=blue!70, fill=blue!20, inner sep=1.2pt},
    neuron2/.style={circle, draw=yellow!70, fill=blue!20, inner sep=1.2pt},
    outneuron/.style={circle, draw=green!50!black, fill=green!20, inner sep=1.2pt}
]

\node[input] (v10) at (0,0) {$x_1$};
\node[input] (v20) at (1.5,0) {$x_2$};
\node[input] (v30) at (3.0,0) {$x_3$};
\node[input] (v40) at (4.5,0) {$x_4$};
\node[input] (v50) at (6.0,0) {$x_5$};

\foreach \x/\id/\lbl in {0.75/11/1, 3.0/21/2, 5.25/31/3} {
  \begin{scope}[shift={(\x,1.8)}]
    \draw[block1] (-0.65,-0.65) rectangle (0.65,0.65);
    \node[neuron1] (n1) at (-0.35, 0.3) {}; \node[neuron1] (n2) at (-0.35, 0) {}; \node[neuron1] (n3) at (-0.35, -0.3) {};
    \node[neuron1] (n4) at (0, 0.2) {}; \node[neuron1] (n5) at (0, -0.2) {};
    \node[neuron1] (n6) at (0.35, 0) {};
    \draw (n1)--(n4); \draw (n2)--(n4); \draw (n2)--(n5); \draw (n3)--(n5); \draw (n4)--(n6); \draw (n5)--(n6);
    \node at (-0.1,-0.95) {$\phi_{\lbl}^{(1)}$};
  \end{scope}
  \node (v\id) at (\x,1.8) {};
}

\foreach \x/\id/\lbl in {1.5/12/1, 4.5/22/2} {
  \begin{scope}[shift={(\x,3.6)}]
    \draw[block2] (-0.65,-0.65) rectangle (0.65,0.65);
    \node[neuron2] (m1) at (-0.35, 0.3) {}; \node[neuron2] (m2) at (-0.35, 0) {}; \node[neuron2] (m3) at (-0.35, -0.3) {};
    \node[neuron2] (m4) at (0, 0.2) {}; \node[neuron2] (m5) at (0, -0.2) {};
    \node[neuron2] (m6) at (0.35, 0) {};
    \draw (m1)--(m4); \draw (m2)--(m4); \draw (m2)--(m5); \draw (m3)--(m5); \draw (m4)--(m6); \draw (m5)--(m6);
    \node at (0,-0.95) {$\phi_{\lbl}^{(2)}$};
  \end{scope}
  \node (v\id) at (\x,3.6) {};
}

\begin{scope}[shift={(3.0,5.4)}]
  \draw[outblock] (-0.65,-0.65) rectangle (0.65,0.65);
  \node[outneuron] (o1) at (-0.35, 0.2){};
  \node[outneuron] (o2) at (-0.35, -0.2){};
  \node[outneuron] (o3) at (0, 0) {}; 
  \node[outneuron] (o4) at (0.35, 0) {};
  \draw[draw=green!50!black] (o1)--(o3); 
  \draw[draw=green!50!black] (o2)--(o3); \draw[draw=green!50!black] (o3)--(o4);
  \node at (0,-0.95) {$\phi_{\mathrm{out}}^{(3)}$};
\end{scope}
\node (vout) at (3.0,5.4) {};

\draw[arrow] (v10) -- (v11); \draw[arrow] (v20) -- (v11);
\draw[arrow] (v20) -- (v21); \draw[arrow] (v30) -- (v21);
\draw[arrow] (v40) -- (v31); \draw[arrow] (v50) -- (v31);
\draw[arrow] (v11) -- (v12); \draw[arrow] (v21) -- (v12);
\draw[arrow] (v21) -- (v22); \draw[arrow] (v31) -- (v22);
\draw[arrow] (v12) -- (vout); \draw[arrow] (v22) -- (vout);
\end{tikzpicture}
\end{subfigure}

\caption{Illustration of a compositional function defined on a DAG (left) and its realization via block-structured neural networks (right), where each node corresponds to a local component.}
\label{fig:dag_block}
\end{figure}

    



We note that each node depends only on a subset of nodes from the previous level, rather than all of them. This sparsity induces a low-dimensional input for each constituent function. The compositional structure is sufficiently general to encompass many models studied in the literature. A fundamental example is the multi-index model $f(x)=g(Ax)$, which underlies classical dimension reduction \citep{li1991sliced,fukumizu2009kernel,klock2021estimating,chen2023kernel}, neural network adaptivity to low-dimensional structure \citep{bach2017breaking,parhi2022near}, and feature learning by gradient-based methods \citep{bietti2023learning,damian2024computational,radhakrishnan2024mechanism}. Another canonical instance is the binary tree compositional structure \citep{poggio2017and,mhaskar2017and}, widely used to demonstrate the expressive and statistical advantages of deep networks over shallow ones. Staircase functions \citep{abbe2022merged}, constructed via iterative compositions that progressively reveal information across layers, also fit naturally within the DAG framework and have been employed to characterize the computational benefits of depth through the notions of information and leap exponents. Moreover, \citet{zhang2025dag} employ a DAG framework to formalize the mathematical structure of reasoning in large language models.

\subsubsection{Approximation rates for compositional functions under norm-constrained NNs}
Next, we present the approximation rates for sparse compositional functions under the following assumption on each node function.
\begin{ass}\label{ass:smooth}
We assume the input $x \in [-1,1]^d$. 
Moreover, for $\ell =1, \ldots,L$, each node function $g_{v^{(\ell)}} \in C^{\alpha_{v^{(\ell)}}}
\big([-1,1]^{d_{\text{in}(v^{(\ell)})}}\big)$ 
with $\alpha_{v^{(\ell)}}=r_{v^{(\ell)}}+\beta_{v^{(\ell)}}$, where $r_{v^{(\ell)}} \in \mathbb{N}_0$ and $\beta_{v^{(\ell)}} \in (0,1]$.
\end{ass}
The above conditions are standard in approximation theory.
The restriction $x \in [-1,1]^d $ is without loss of generality, as any bounded domain can be rescaled accordingly. 
The regularity and boundedness assumptions are imposed at the level of individual node functions and are loose since only requires $\alpha_{v^{(\ell)}}>0$, whose input dimensions $d_{{\mathrm{in}}(v^{(\ell)})}$ may vary across the graph. In fact, The assumption that each node function is valued in the unit ball is made only for convenience and can be relaxed to any bounded range; see Appendix~\ref{subsec:proof_apx_cmp} for the rescaling argument.
The following theorem quantifies the approximation rate of sparse compositional functions in terms of node-level smoothness and dimensionality. The proof appears in Appendix~\ref{subsec:proof_apx_cmp}.
\begin{theorem}\label{thm:apx_cmp}
Let $f:[-1,1]^d \to \R$ be a sparse compositional function associated with a DAG
$(\mathcal V,\mathcal E)$ with $L$ levels as defined in \eqref{equ:cmp}.
Under Assumption~\ref{ass:smooth}, for any $K \ge 1$, there exists a neural network $\phi \in \CN\CN(W,D,K)$
with depth
\(
D = 2L \max_{\ell, v^{(\ell)}} \left\lceil \log_2\!\big(d_{\text{in}(v^{(\ell)})} + r_{v^{(\ell)}}\big) \right\rceil + 2L,
\)
such that, for any width
\(
W \ge \max_{\ell, v^{(\ell)}} 
K^{\frac{2 d_{\text{in}(v^{(\ell)})} + \alpha_{v^{(\ell)}}}{(D+L)\, d_{\text{in}(v^{(\ell)})} + 2L}},
\)
the approximation error satisfies 
\[
\|f - \phi\|_\infty \;\lesssim\;
\max_{v^{(\ell)}:\pi_{1 \to L}}
K^{-\frac{2\alpha_{v^{(\ell)}}^*}{2L + (D+L)\, d_{\text{in}(v^{(\ell)})}}},
\]
where $\pi_{1\to L}$ denotes the critical path selected by the error propagation in \eqref{equ:path} from level $1$ to level $L$, and $\alpha_{v^{(\ell)}}^* := \alpha_{v^{(\ell)}} \prod_{m=\ell+1}^{L} \min\{\alpha_{v^{(m)}}, 1\}$ is the effective regularity index along this path.
\end{theorem}

Some remarks are in order regarding the theorem.

\paragraph{Network construction}
The depth $D$ grows only logarithmically with the local structural quantities, up to the factor $L$ caused by the compositional structure. Hence, the
construction uses relatively shallow but sufficiently wide subnetworks. This is
important under a product norm constraint, since the approximation exponent
deteriorates with depth. The logarithmic-depth construction therefore keeps the
depth-dependent cost mild while allowing width to absorb the required
approximation complexity. The width $W$ can be taken arbitrarily large, so the result naturally covers the overparameterized regime. The construction crucially exploits the composition and concatenation properties of norm-constrained networks (Proposition~\ref{lem:elementary} in Appendix~\ref{subsec:pf_intro_NN_F}), allowing us to approximate node functions across different levels of the DAG with subnetworks (see the right panel of Figure~\ref{fig:dag_block}). 

\paragraph{Approximation rate}
The approximation rate is characterized explicitly in terms of the norm constraint $K$, rather than the parameter count. The path-wise maximum can be further upper-bounded by the worst case over all possible paths. The exponent depends on the singularity index $\alpha_{v^{(\ell)}}^*$ and the local input dimension $d_{\text{in}(v)}$ of the most restrictive node function, instead of the global smoothness $\alpha^*$ and ambient dimension $d$. The index $\alpha_{v^{(\ell)}}^*$ reflects the iterative accumulation inherent in H\"older compositions, a phenomenon also observed in \citep{juditsky2009nonparametric,schmidt2020nonparametric}. If the node functions are at least Lipschitz continuous, i.e., $\alpha_{v^{(\ell)}} \ge 1$, then the singularity index satisfies $\alpha_{v^{(\ell)}}^* = \alpha_{v^{(\ell)}}$, so the smoothness is inherited directly from each node rather than arising through iterative accumulation. Compared with Theorem~\ref{thm:apx_s_F}, the replacement of $d$ by the local input dimension $d_{\text{in}(v)}$ highlights the adaptivity of deep networks to sparse compositional structure and clarifies how hierarchical representations alleviate the curse of dimensionality. The compositional level $L$ contributes to the rate through both the
 depth $D$ and the denominator
in the approximation exponent. This reflects the
cost of propagating approximation errors through multiple compositions. Nevertheless, the dependence on $L$ is substantially
milder than the ambient-dimensional dependence in classical approximation
rates.  We provide explicit rates in Remark~\ref{remark:appx} to make this phenomenon concrete. However, the optimality of these rates in the norm-constrained setting is unclear, and establishing matching lower bounds remains an important direction for future work.


\paragraph{Comparison with prior work}
To the best of our knowledge, this is the first result establishing approximation rates for sparse compositional functions in terms of a norm constraint. In contrast, existing works analyze compositional function classes with guarantees expressed in terms of parameter counts $M$. In particular, \citep{schmidt2020nonparametric,bauer2019deep,kohler2021rate} study compositional H\"older spaces using deep neural networks and obtain rates of order $\max_v M^{-\alpha_v^*/d_{\text{in}(v)}}$. The same rate is derived for ridge functions (single-index models) with Lipschitz continuity using deep CNNs \citep{zhou2020theory,fang2020theory}. Moreover, \citet{mao2023approximating} investigate compositional polynomial spaces with deep CNNs. While all these works demonstrate that the rate depends exponentially on intermediate latent dimensions rather than the ambient dimension, their results are formulated exclusively in terms of parameter counts.

To clarify how the rate scales with dimension, we instantiate the bounds with representative choices of the level $L$ studied in prior work.
\br[Explicit approximation rates examples] \label{remark:appx}
We consider the following three examples.

\textbf{(Multi-index model)}
Consider the multi-index model $f(x)=g(Ax)$ with $A\in\mathbb{R}^{s\times d}$ and $s<d$ \citep{li1991sliced}. 
This induces a two-level compositional structure with $L=2$, where 
$d_{\text{in}(v^{(1)})}=d$ for all $v^{(1)}\in \CV^{(1)}$ (linear projection nodes) and 
$d_{\text{in}(v^{(2)})}=s$ for the nonlinear node. 
Although this structure is not sparse compositional, it still fits within our framework.
The approximation rate becomes
\[
\max_{v^{(\ell)}:\pi_{1 \to 2}}K^{-\frac{2\alpha^*_{v^{(\ell)}}}{4 + \left(4\left\lceil \log_2(s+r) \right\rceil+6\right)d_{\text{in}(v^{(\ell)})}}}
\lesssim
K^{-\frac{\alpha}{C_1 s}},
\]
where $\alpha=r+\beta$ denotes the regularity of $g$ and
\(
C_1 = 2\left\lceil \log_2(s+r) \right\rceil+5
\)
is independent of the ambient dimension $d$. 
Thus, the rate depends solely on the intrinsic dimension $s$, demonstrating the adaptivity of overparameterized networks.

\textbf{(Binary tree)}
   In the binary tree structure \citep{poggio2017and}, $L=\log_2 d$ and $d_{\text{in}(v)}=2$.
The approximation bound is
\[
\|f-\phi\|_\infty\le C_{\mathcal G}^{\mathrm{app}}\max_{v^{(\ell)}:\pi_{1\to L}}K^{-\alpha_{v^{(\ell)}}^*/(C_2\log_2d)},
\]
where $C_2=2\max_{v\in\mathcal V}\lceil\log_2(2+r_v)\rceil+4$ is independent of $d$. The exponent depends on $d$ only through $\log_2d$. For fixed bivariate smoothness bounds, let $\overline D=2\max_v\lceil\log_2(2+r_v)\rceil+2$ and $\alpha_{\max}=\max_v\alpha_v$. The constant $C_{\mathcal G}^{\mathrm{app}}$ can be bounded by
\[
C_{\mathcal G}^{\mathrm{app}}\le C_0(\log_2d)(d/2+1)^{\rho},\qquad \rho=\frac{\alpha_{\max}(\overline D+1)}{2\overline D+4},
\]
where $C_0$ depends on the fixed local smoothness and normalization bounds. The remaining dependence is polynomial in $d$. Overall, the explicit dependence on $d$ is polynomial in the prefactor and logarithmic in the denominator of the rate exponent, with no exponential factor in $d$.

\textbf{(General constant level)} 
Alternatively, one may take $L=c$ as a constant independent of $d$ \citep{schmidt2020nonparametric}. Then the rate scales as
\[
\max_{v^{(\ell)}:\pi_{1 \to L}} K^{-\frac{2\alpha^*_{v^{(\ell)}}}{2c + \left(2c\max_{v}\left\lceil \log_2({d_{\text{in}(v})}+{r_{v}}) \right\rceil+3c\right)d_{\text{in}(v^{(\ell)})}}}
\lesssim 
\max_{v^{(\ell)}:\pi_{1 \to L}} K^{-\frac{\alpha^*_{v^{(\ell)}}}{C_3d_{\text{in}(v^{(\ell)})}}},
\]
where 
\(
C_3 = c\max_{v}\left\lceil \log_2({d_{\text{in}(v})}+{r_{v}}) \right\rceil+2.5
\)
is also independent of $d$. In this regime, the rate depends only on the intermediate dimensions $d_{\text{in}(v)}$, which are typically small, and is independent of the ambient dimension $d$.
In this regime, the displayed exponent depends on the local dimensions, while the prefactor $C_{\mathcal G}^{\mathrm{app}}$ still records the size and structure of the DAG.
\er

\section{Statistical analysis for sparse compositional functions}\label{sec:statistical}

The goal of supervised learning with least squares loss is to estimate a minimizer of the expected risk
\begin{equation*}
\CR(f):=\EE_\rho[(f(x)-y)^2]
= \int_{X \times \R} (f(x)-y)^2 \md \rho(x,y)
\end{equation*}
over all measurable functions, based on a training sample set $\{(x_i, y_i)\}_{i=1}^n \overset{\text{i.i.d.}}{\sim} \rho^n$. Here $\rho$ is the unknown joint probability distribution of $X \times \R$.  The quality of a learning solution $\hat f$ is measured by the 
 {excess risk} \[\CR(\hat{f})-\CR(f_*),\] where  $f_*$ denotes a minimizer of $\CR(f)$. It quantifies the statistical error of $\hat f$ relative to the optimal predictor under the population distribution $\rho$, rather than on the training sample.
In this section, we establish excess risk bounds for the empirical risk minimizer $\hat{f}$ defined via empirical risk minimization
\begin{equation}\label{equ:erm}
   \hat{f} := \mathop{\arg\min}_{f \in \CN\CN_B(W, D, K)} \widehat{\CR}(f),
\quad
\widehat{\CR}(f) = \frac{1}{n} \sum_{i=1}^n (f(x_i) - y_i)^2. 
\end{equation}
Here, $\CN\CN_B(W, D, K)$ denotes neural networks with outputs uniformly bounded by $B>0$. 
Such realizations are obtained by composing the network from Theorem~\ref{thm:apx_cmp} with the clipping operator $\chi_B:\mathbb{R}\to[-B,B]$, given by
\(
\chi_B(x)=\sigma(x)-\sigma(-x)-\sigma(x-B)+\sigma(-x-B).
\)
This ensures bounded outputs and enables moment control; without loss of generality, we take $B=1$.

\begin{ass}\label{ass:boundeddata}
The input space $\mathcal X \subset \mathbb{R}^d$ satisfies $\|x\|\le 1$, and the output satisfies $|y|\le B$ for some $B>0$. Without loss of generality, we may assume $B=1$ by rescaling.
\end{ass}
The input constraint $\|x\|\le 1$ can be enforced by rescaling for bounded inputs. The boundedness of the outputs is standard and ensures moment control for applying concentration inequalities. 

The following theorem establishes the sample complexity of learning sparse compositional H\"older functions with neural networks under Frobenius norm constraints. The proof is provided in Appendix~\ref{sec:pf_stats}.
\begin{theorem}\label{thm:excess_rate}
Let $f_*:[-1,1]^d \to \R$ be a sparse compositional function defined by \eqref{equ:cmp}, associated with a DAG $(\mathcal V,\mathcal E)$ with $L$ levels.  
Under Assumptions~\ref{ass:smooth} and \ref{ass:boundeddata}, for any $n \in \NN_+$, there exists an estimator $\hat{f} \in \CN\CN_1(W,D,K)$ with depth
\(
D = 2L \max_{\ell, v^{(\ell)}} \left\lceil \log_2\!\big(d_{\text{in}(v^{(\ell)})} + r_{v^{(\ell)}}\big) \right\rceil + 2L ,
\)
such that, for any width
\(
W \ge \max_{\ell, v^{(\ell)}} 
K^{\frac{2 d_{\text{in}(v^{(\ell)})} + \alpha_{v^{(\ell)}}}{(D+L)\, d_{\text{in}(v^{(\ell)})} + 2L}},
\)
and norm constraint
\begin{equation}\label{equ:norm_bound}
K \asymp
\max_{v^{(\ell)}:\pi_{1 \to L}}
\left(n\right)^{
\frac12
\frac{2L+(D+L)d_{\text{in}(v^{(\ell)})}}{
2L+(D+L)d_{\text{in}(v^{(\ell)})}+4\alpha_{v^{(\ell)}}^*}},
\end{equation}
the excess risk satisfies, with probability at least $1 - \delta$,
\[
\CR(\hat f)-\CR(f_*)
\;\lesssim\;
\sqrt{\log \frac{4}{\delta}}
\max_{v^{(\ell)}:\pi_{1 \to L}}n^{-
\frac{2\alpha_{v^{(\ell)}}^*}{
2L+(D+L)d_{\text{in}(v^{(\ell)})}+4\alpha_{v^{(\ell)}}^*
}},
\]
where $\pi_{1\to L}$ denotes the critical path selected by the error propagation in \eqref{equ:path} from level $1$ to level $L$, and $\alpha_{v^{(\ell)}}^* := \alpha_{v^{(\ell)}} \prod_{m=\ell+1}^{L} \min\{\alpha_{v^{(m)}}, 1\}$ is the effective regularity index along this path.
\end{theorem}
To the best of our knowledge, this is the first result establishing sample complexity guarantees for estimators generated by neural networks that cover the overparameterized regime for sparse compositional function classes. 
The depth $D$ scales linearly with the number of compositional levels $L$, while the width $W$ can be arbitrarily large beyond a minimal threshold, placing the model in the overparameterized regime. 
The learning rate is governed by the most restrictive node along the critical path and can be further upper-bounded by the worst-case node over all possible paths. 
The resulting rate depends on local structural complexity rather than the ambient dimension, thereby mitigating the curse of dimensionality. 
More explicit dimension dependence for representative compositional models is provided in Remark~\ref{remark:excess_risk}.
We further discuss this theorem below.


\paragraph{Scaling of the norm constraint $K$}
The choice of $K$ in \eqref{equ:norm_bound} follows from balancing approximation and estimation errors, corresponding to bias induced by hypothesis class restriction and variance arising from finite samples.
This trade-off is formalized via a standard error decomposition.
Define 
\(
\tilde{f} := \mathop{\arg\min}_{f \in \CN\CN_B(W, D, K)} \CR(f),
\)
then the excess risk
\begin{equation}\label{equ:errdcmp}
\CR(\hat{f}) - \CR(f_*) \leq 2 \sup_{f \in \CN\CN_1(W, D, K)} \left|  \CR(f)   -  \widehat{\CR}(f) \right| + \CR(\tilde{f}) - \CR(f_*).
\end{equation}
Equation \eqref{pf_errdcmp} provides its proof for completeness. The first term captures the estimation (sample) error, whereas the second term represents the approximation error.
Theorem~\ref{thm:apx_cmp} shows that the approximation error decays polynomially as $K^{-\gamma}$, where $\gamma$ depends on the singularity index $\alpha_v^*$ and the local input dimensions $d_{\mathrm{in}(v)}$. 
This rate decreases as $K$ increases, since a larger $K$ corresponds to a richer hypothesis class to approximate the target function.
The empirical Rademacher complexity $\widehat{\mathfrak{R}}_n(\CN\CN(W,D,K))$ provides a tool for controlling the estimation error. 
%
Indeed, \citet{golowich2020size} show that
\(
\widehat{\mathfrak{R}}_n(\CN\CN(W,D,K))\lesssim
K/\sqrt{n}.
\)
Related norm-based bounds have also been studied in \citep{neyshabur2015norm,bartlett2017spectrally,galanti2023norm,poggio2024compositional}.
Balancing $K^{-\gamma}$ and $K/\sqrt{n}$ yields the optimal scaling $K \asymp n^{1/(2(\gamma+1))}$, leading to the stated excess risk rate. 
Thus, $K$ acts as a continuous capacity parameter, analogous to model size in classical bias–variance tradeoffs but better suited to the overparameterized regime.


\paragraph{Comparison with parameter-count-based bounds and beyond}
Classical excess risk analyses based on parameter count $M$ balance approximation and estimation errors of the form
\( M^{-\tau} \) and \( {M}/{n} \), where $\tau>0$ depends on the singularity index and local input dimensions. 
Balancing these two terms yields the optimal choice $M \asymp n^{1/(1+\tau)}$, which is smaller than the sample size $n$, and leads to the rate
\(
n^{-\tau/(1+\tau)}.
\)  Variants of this approximation–estimation tradeoff appear throughout the theory of deep ReLU networks for nonparametric regression and compositional function classes \citep{bauer2019deep,schmidt2020nonparametric,kohler2022estimation,han2020depth}, as well as in theoretical analyses of deep convolutional networks based on approximation, consistency, and non-asymptotic excess-risk bounds \citep{zhou2020universality,feng2023generalization,kohler2023analysis}. Having compared approximation rates in the previous section, we now examine the estimation error in more detail. 
A key limitation of parameter-count-based bounds is that when $M > n$, the term $M/n$ no longer vanishes, and thus fails to capture the overparameterized regime \citep{bartlett2002rademacher,belkin2019reconciling}. 
Moreover, achieving fast rates of order $n^{-1}$ typically requires refined techniques such as local Rademacher complexity and variance-dependent bounds \citep{bartlett2005local,steinwart2008support,wainwright2019high}. 
This raises the question of whether similar refined analyses can be developed for norm-constrained neural networks. While such results exist for infinite-width shallow models \citep{yang2024nonparametric,siegel2024sharp}, extending localized fast-rate analyses to deep norm-constrained networks in the overparameterized regime remains largely open. 
Our rates exhibit the same qualitative dependence on smoothness and intrinsic dimension as classical results, but their optimality in the norm-constrained setting is still unclear. {For a direct comparison, the classical minimax squared-risk benchmark for a $d$-dimensional H\"older ball is $n^{-2\alpha/(2\alpha+d)}$ \citep{gyorfi2006distribution}. Specializing Theorem~\ref{thm:excess_rate} to one compositional level gives the slower upper bound $n^{-2\alpha/\{2+d(D+1)+4\alpha\}}$. For compositional classes, parameter-count analyses obtain rates governed by $n^{-2\alpha_v^*/(2\alpha_v^*+d_{\mathrm{in}(v)})}$, up to logarithmic factors and under their respective assumptions \citep{schmidt2020nonparametric,kohler2021rate}. Our contribution is the control of approximation by a product norm in a class permitting arbitrarily large width; the estimation argument uses existing Rademacher bounds. Neither part is known to be sharp, and the resulting upper bound is not a minimax characterization. Matching lower bounds and localized complexity analyses remain open.}

\br[Explicit rates in representative settings]\label{remark:excess_risk}
We specialize the excess risk bound to several representative compositional structures considered above with the same constants in Remark~\ref{remark:appx}. 

\textbf{(Multi-index model)}
For the multi-index model $f(x)=g(Ax)$ with intrinsic dimension $s$ ($L=2$), the rate becomes 
\[ n^{-\frac{1}{2}\frac{2\alpha}{C_1 s +2 \alpha}}, \]
where $\alpha$ is the regularity of $g$ and $C_1 $ is independent of the ambient dimension $d$. 

\textbf{(Binary tree)}
For the binary tree structure with $L=\log_2 d$ and $d_{\text{in}(v)}=2$, the rate is 
\[ \max_{v^{(\ell)}:\pi_{1 \to L}} n^{-\frac{1}{2}\frac{2\alpha^*_{v^{(\ell)}}}{C_2 \log_2 d + 2\alpha^*_{v^{(\ell)}}}}, \]
so the dimensional dependence enters only through $\log_2 d$. 

\textbf{(General constant level)}
For constant level $L=c$, the rate becomes 
\begin{equation*}
    \max_{v^{(\ell)}:\pi_{1 \to L}} n^{-\frac{1}{2}\frac{2\alpha^*_{v^{(\ell)}}}{C_3 d_{\text{in}(v^{(\ell)})} + 2\alpha^*_{v^{(\ell)}}}}.
\end{equation*}

In all cases, the rates depend on the local dimension, indicating that hierarchical representations mitigate the curse of dimensionality even in the overparameterized regime.
\er

\subsection{Discussion}\label{sec:discussion}
\paragraph{Compositional and mixed smoothness perspectives}
Classical approximation and excess risk rates, $M^{-\alpha/d}$ and $n^{-2\alpha/(2\alpha + d)}$, depend on the interplay between smoothness $\alpha$ and dimension $d$ \citep{pinkus1999approximation,gyorfi2006distribution}, reflecting the {blessing of smoothness} and the {curse of dimensionality} \citep{donoho2000high}. This motivates two approaches to adaptivity in high dimensions: exploiting structure in dimensionality or in regularity. We focus on the former via the compositional principle, showing the rates depend on local dimensionality (see Remarks~\ref{remark:appx} and~\ref{remark:excess_risk}).
A complementary line studies function spaces with mixed derivatives, which limit high-order interactions across variables \citep{montanelli2019new,suzuki2018adaptivity}. Most results in this direction are based on parameter-count analyses and demonstrate mitigation of the curse of dimensionality. It remains unclear whether norm-constrained deep networks can exploit such structure or achieve generalization guarantees.

\paragraph{Explicit and implicit regularization}
The analysis in this work is based on norm-constrained empirical risk minimization \eqref{equ:erm}, corresponding to an {Ivanov regularization} scheme. 
In contrast, modern neural networks are often trained in the overparameterized regime without explicit constraints, achieving interpolation while still generalizing well, a phenomenon known as {benign overfitting} \citep{bartlett2020benign,belkin2021fit}. This behavior is commonly interpreted through the lens of implicit regularization, which can be idealized as
\(
\min_{\theta} \ \kappa(\theta)\) such that \( f_\theta(x_i)=y_i,
\)
see, e.g., \citep{neyshabur2014search,soudry2018implicit,gunasekar2018implicit}.
A classical alternative is {Tikhonov regularization},
\(
\min_{f \in \CN\CN(W,D)} \widehat{\CR}(f) + \lambda \kappa(\theta),\lambda > 0,
\)
which introduces a norm penalty in the objective and provides a natural bridge between explicit regularization and implicit bias. 
Understanding the convergence rates of Tikhonov-regularized estimators may provide a principled route to characterizing minimal-norm interpolation and implicit regularization in deep networks. Establishing such connections in overparameterized settings remains largely open.

\section{Conclusion}\label{sec:conclusion}
We provide a norm-based perspective on the compositional principle in deep neural networks, showing that both approximation and generalization can be characterized directly in terms of Frobenius norm constraints, beyond classical parameter-count-based analyses. 
Our results demonstrate that, for sparse compositional functions represented by DAGs, the learning rates are governed by local smoothness and intrinsic input dimensions along critical paths. In particular, the dependence on the ambient dimension is replaced by logarithmic or intermediate dimensions in representative compositional models, thereby mitigating the CoD. 
These findings extend to the overparameterized regime, where network width can be arbitrarily large beyond a minimal threshold. Overall, this work highlights that norm-constrained deep networks can exploit compositional structure to achieve statistically efficient learning in high dimensions.

\section*{Acknowledgments}
L. R. acknowledges the financial support of the European Commission (Horizon Europe grant ELIAS 101120237), and the Ministry of Education, University and Research (FARE grant ML4IP R205T7J2KP). T. P. acknowledges the support from AFOSR project (FA955024-1-0231).

{\small
\bibliographystyle{plainnat}
\bibliography{biblio}
}

\appendix

\section{Basic properties of Frobenius norm-constrained neural networks}\label{subsec:pf_intro_NN_F}

For any function $f\in \CN\CN(W,D,K)$, the following proposition shows that all intermediate layers can be rescaled with a unit Frobenius norm,
while the Frobenius norm of the final layer remains bounded by $K$, without changing the
realized function.
This property follows from the positive homogeneity of the ReLU activation and simplifies the subsequent proofs.

\begin{proposition}[Rescaling with Frobenius norm]\label{prop:rescaling}
Every $\phi \in  \mathcal{N}\mathcal{N}(W,D,K)$ can be written in the form \eqref{equ:relu_network} such that
\[
\|A_D\|_F \le K,
\quad
\sqrt{\|A_\ell\|_F^2 + \|b_\ell\|_2^2} \le 1,
\quad \text{for all } \ell = 0,\ldots,D-1.
\]
\end{proposition}

\begin{proof}[Proof of Proposition~\ref{prop:rescaling}]
Let $\phi\in\mathcal{N}\mathcal{N}(W,D,K)$ have parameters $\theta=((A_\ell,b_\ell)_{\ell=0}^{D-1},A_D)$ satisfying
\[
\kappa(\theta)=\|A_D\|_F\prod_{\ell=0}^{D-1}\sqrt{\|A_\ell\|_F^2+\|b_\ell\|_2^2+1}\le K.
\]
Set $s_{-1}:=1$ and recursively define
\[
s_\ell:=\sqrt{s_{\ell-1}^2\|A_\ell\|_F^2+\|b_\ell\|_2^2+1},\qquad \ell=0,\ldots,D-1.
\]
In particular, $s_\ell\ge1$. Define the rescaled parameters by
\[
\widehat A_\ell:=\frac{s_{\ell-1}}{s_\ell}A_\ell,\qquad \widehat b_\ell:=\frac{b_\ell}{s_\ell},\qquad \ell=0,\ldots,D-1,
\]
and set $\widehat A_D:=s_{D-1}A_D$. These rescalings preserve the width and depth of the network.

To verify that the realized function is unchanged, observe that positive homogeneity of ReLU gives, for every vector $z$ of the appropriate dimension,
\[
\sigma\left(\widehat A_\ell\frac{z}{s_{\ell-1}}+\widehat b_\ell\right)=\sigma\left(\frac{A_\ell z+b_\ell}{s_\ell}\right)=\frac{1}{s_\ell}\sigma(A_\ell z+b_\ell).
\]
Since $s_{-1}=1$, applying this identity successively through the hidden layers shows that the output of the last hidden layer is divided by $s_{D-1}$. This factor is canceled by $\widehat A_D=s_{D-1}A_D$, so the rescaled network realizes the same function $\phi$.

For every $\ell=0,\ldots,D-1$, the definition of $s_\ell$ yields
\[
\|\widehat A_\ell\|_F^2+\|\widehat b_\ell\|_2^2=\frac{s_{\ell-1}^2\|A_\ell\|_F^2+\|b_\ell\|_2^2}{s_\ell^2}=1-\frac{1}{s_\ell^2}\le1.
\]
Moreover, because $s_{\ell-1}\ge1$,
\[
s_\ell^2=s_{\ell-1}^2\|A_\ell\|_F^2+\|b_\ell\|_2^2+1\le s_{\ell-1}^2(\|A_\ell\|_F^2+\|b_\ell\|_2^2+1).
\]
Taking square roots and iterating from $s_{-1}=1$, we obtain
\[
s_{D-1}\le\prod_{\ell=0}^{D-1}\sqrt{\|A_\ell\|_F^2+\|b_\ell\|_2^2+1}.
\]
Consequently,
\[
\|\widehat A_D\|_F=s_{D-1}\|A_D\|_F\le\kappa(\theta)\le K.
\]
Thus, the rescaled parameters satisfy all the asserted norm bounds.
\end{proof}

The following proposition collects several elementary properties of neural networks under Frobenius norm constraints that will be used repeatedly throughout the paper. Analogous properties for alternative choices of $\kappa(\theta)$ have been established in prior work, such as \citep{schmidt2020nonparametric,jiao2023approximation}. 
\begin{proposition}[Elementary properties of $\CN\CN(W,D,K)$]  \label{lem:elementary}
Let $\phi_1 \in  \mathcal{N}\mathcal{N}(W_1,D_1,K_1)$ and $\phi_2 \in  \mathcal{N}\mathcal{N}(W_2,D_2,K_2)$.
\begin{enumerate}[label=(\theproposition.\arabic*),ref=\theproposition.\arabic*] 
 \item (Linear combination)  Suppose  $ D_1 = D_2=D$,
for any constants $c_1, c_2 \in \mathbb{R}$, the function 
\(
\phi = c_1 \phi_1 + c_2 \phi_2
\)
satisfies
\[
\phi \in \mathcal{N}\mathcal{N}\left(W_1+W_2, D, (\sqrt{3})^{D}\sqrt{(c_1 K_1)^2 + (c_2 K_2)^2} \right).
\]
More generally,  the linear combination of $N$ neural networks with same depth satisfies\\ $\phi \in \CN\CN\left(\sum_{i=1}^NW_i, D, \sqrt{N+1}^D \sqrt{\sum_{i=1}^N (c_iK_i)^2}\right)$.  \label{lem:linearcmb}
\item (Concatenation)
If  $ D_1 = D_2=D$, define the concatenation $\phi(x) := (\phi_1(x), \phi_2(x))$, 
then 
\[
\phi \in NN
\big(W_1 + W_2,  D,  (\sqrt{3})^{D}\sqrt{K_1^2 + K_2^2}\big).
\]
More generally, the concatenation for $N$ neural networks with same depth satisfies \\ $\phi \in \mathcal{N}\mathcal{N} \left( \sum_{i=1}^{N} W_i,  D,  \sqrt{N+1}^{D}\sqrt{\sum_{i=1}^{N} K_i^2}   \right)$. \label{lem:concate}
\item (Composition)
If $d_2=k_1$, then the composition \[\phi=\phi_2 \circ \phi_1 \in  \mathcal{N}\mathcal{N}\left(\max\{W_1,W_2\},D_1+D_2,\sqrt{2}^{D_1} K_2\sqrt{K_1^2+2}\right).\]
More generally, the composition of $N$ neural networks with aligned input and output dimensions satisfies $\phi\in \CN\CN\left( \max_{i}\{W_i\}, \sum_{i=1}^N D_i, \sqrt{2}^{\sum_{i=1}^{N-1}D_i}K_N\prod_{i=1}^{N-1} \sqrt{K_i^2+2}\right)$. \label{lem:comp}
\end{enumerate}    
\end{proposition}

Since the Frobenius norm constraint is defined through the product over $D$ layers, the factors involving exponents of $D$ arise naturally. This also explains why our constructions favor shallower but wider networks.

\begin{proof}[Proof of Proposition~\ref{lem:elementary}] 

(Proof of Proposition~\ref{lem:linearcmb})
Let $\phi_1 \in  \mathcal{N}\mathcal{N}(W_1,D,K_1)$ and $\phi_2 \in  \mathcal{N}\mathcal{N}(W_2,D,K_2)$ be two ReLU networks of the same depth $D$.
Write their realizations as
\[
\phi_i(x) = A_D^{(i)} \sigma\!\big( A_{D-1}^{(i)} \sigma( \cdots \sigma( A_0^{(i)} x + b_0^{(i)} ) \cdots ) + b_{D-1}^{(i)} \big),
\quad i=1,2.
\]
By Proposition~\ref{prop:rescaling}, we have
\[
\| A^{(i)}_\ell\|_F^2+\|b_\ell^{(i)}\|_2^2 \le 1,
\quad \ell=0,\ldots,D-1.
\]

We construct a network that computes $\phi_1$ and $\phi_2$ in parallel.
For each hidden layer $\ell=0,\ldots,D-1$, define the block-diagonal augmented matrix
\begin{equation}\label{equ:linear2}
   A_\ell
=
\begin{pmatrix}
 A^{(1)}_\ell & 0 \\
0 & A^{(2)}_\ell
\end{pmatrix}, \quad
b_\ell= \begin{pmatrix}
    b^{(1)}_\ell \\
    b^{(2)}_\ell
\end{pmatrix}. 
\end{equation}
This network has width $W_1+W_2$ and depth $D$, and satisfies
\[
\| A_\ell\|_F^2+\|b_\ell\|_2^2+1
=
\| A_\ell^{(1)}\|_F^2+\|b_\ell^{(1)}\|_2^2 
+
\| A_\ell^{(2)}\|_F^2+\|b_\ell^{(1)}\|_2^2+1
\le
3.
\]

At the output layer, define
\[
A_D := \big(c_1 A^{(1)}_D,\; c_2 A^{(2)}_D\big),
\]
so that the resulting network computes
\(
\phi(x)=c_1\phi_1(x)+c_2\phi_2(x).
\)
Moreover,
\[
\|A_D\|_F^2
\leq
(c_1K_1)^2 + (c_2K_2)^2.
\]
Combining the above bounds over all layers yields
\[
\phi \in NN \Big(W_1+W_2,   D,  
(\sqrt{3})^{D}\sqrt{(c_1K_1)^2+(c_2K_2)^2}\Big).
\]
Moreover, the linear combination of $N$ neural networks yields hidden-layer
matrices $A_\ell$ and bias vectors $b_\ell$ with $N$ blocks as~\eqref{equ:linear2}. Combining this
block construction with the rescaling proposition, we have
$\|A_\ell\|_F^2+\|b_\ell\|_2^2+1 \le N+1$. The output matrix is given by
$A_D := \big(c_1 A_D^{(1)},\ldots,c_N A_D^{(N)}\big)$. Combining these
estimates with the definition of $\kappa(\theta)$ completes the proof.

(Proof of Proposition~\ref{lem:concate}) We present the proof for the concatenation of two neural networks; the general case of \(N\) networks follows by the same argument. construct a network $\phi=(\phi_1,\phi_2)$ by running the two networks in parallel.
For each hidden layer $\ell=0,\ldots,D-1$, the  augmented matrix is block-diagonal and given by
\[
 A_\ell
:=
\begin{pmatrix}
A^{(1)}_\ell & 0\\
0 &  A^{(2)}_\ell
\end{pmatrix}, \quad
b_\ell= \begin{pmatrix}
    b^{(1)}_\ell \\
    b^{(2)}_\ell
\end{pmatrix}
\]
and the final layer is
\[
A_D := \begin{pmatrix}  
A^{(1)}_D &0 \\ 0& A^{(2)}_D \end{pmatrix}.
\]
The resulting network has depth $D$ and width $W_1+W_2$, and clearly realizes
$\phi(x)=(\phi_1(x),\phi_2(x))$.

Moreover, for $\ell=0,\ldots,D-1$, the rescaling proposition gives
\[
\| A_\ell\|_F^2 + \|b_\ell\|_2^2+1
=
\| A^{(1)}_\ell\|_F^2+\|b_\ell^{(1)}\|_2^2+\| A^{(2)}_\ell\|_F^2
+\|b_\ell^{(2)}\|_2^2+1\le 3.
\]
For the last layer,
\[
\| A_D\|_F^2 = \| A^{(1)}_D\|_F^2+\|A^{(2)}_D\|_F^2\le K_1^2 +K_2^2.
\]
Therefore,
\[
\| A_D\|_F\prod_{\ell=0}^{D-1}\sqrt{\| A_\ell \|_F^2 +\|b_\ell\|_2^2} 
\le
(\sqrt{3})^{D}\sqrt{K_1^2+K_2^2}.
\]
The extension to $N$ networks follows by the same parallel construction with an $N$-block
diagonal matrix at each layer, giving the factor $(\sqrt{N+1})^{D}$ and the bound
$K=(\sqrt{N+1})^{D}\sqrt{\sum_{i=1}^N K_i^2}$.    

(Proof of Proposition~\ref{lem:comp})
Assume $d_2 = k_1$ so that the composition $\phi = \phi_2 \circ \phi_1$ is well-defined.
We construct the composed network by stacking the layers of $\phi_2$ after those of $\phi_1$. The resulting network has depth $D_1 + D_2$ and width at most $\max\{W_1,W_2\}$.
For $\ell=0,\ldots,D_1-1$, we keep the layers of $\phi_1$, and append the layers of $\phi_2$ afterward.
The key step is the interface layer. At the transition, the mapping becomes
\[
x \mapsto \sigma\big(A_0^{(2)} \phi_1(x) + b_0^{(2)}\big).
\]
At layer $D_1$, the augmented matrix will be
\[
\widetilde{A}_{D_1}=A^{(2)}_0A^{(1)}_{D_1}, \quad \tilde b_{D_1}=b_0^{(2)} .
\]
Note that $\|\widetilde{A}_{D_1}\|_F^2+\|\widetilde{b}_{D_1}\|_2^2+1=\left\|\left( A_0^{(2)} A_{D_1}^{(1)},\, b_0^{(2)},1 \right)\right\|_F^2$ can be estimated by
\[
\left\|\left( A_0^{(2)},\, b_0^{(2)},1 \right)
\begin{pmatrix}
A_{D_1}^{(1)} & 0 &0 \\
0 & 1 &0\\
0&0&1
\end{pmatrix}
\right\|_F^2
\le
\left\|\left( A_0^{(2)},\, b_0^{(2)},1 \right)\right\|_F^2
\left\|
\begin{pmatrix}
A_{D_1}^{(1)} & 0 &0\\
0 & 1 &0\\
0 & 0& 1
\end{pmatrix}
\right\|_F^2.
\]
By Proposition~\ref{prop:rescaling}, 
the Frobenius norm of the composed network satisfies
\begin{align*}
\kappa(\theta)
&\le
\left(\prod_{\ell=0}^{D_1-1} \sqrt{\|A_\ell^{(1)}\|_F^2+\|b_\ell^{(1)}\|_2^2+1}\right)
\cdot
\sqrt{\|A_{D_1}^{(1)}\|_F^2 + 2} \cdot
K_2 \\
&
\le 2^{D_1/2} K_2 \sqrt{K_1^2 + 2}.
\end{align*}
The general case of $N$ compositions follows by induction and it completes the proof.

\end{proof}

\section{Approximation ability of Frobenius norm-constrained neural networks}

\subsection{Proof of Subsection~\ref{subsec:apx_s_H}}\label{subsec:pf_apx_s_H}
Based on the elementary properties of Frobenius norm-constrained neural networks, we are now ready to establish the approximation rate for H\"older spaces, starting with the approximation of the square function. 
\begin{lemma}[Approximation of $x^2$]
\label{lem:approx_x2}
For any $k \in \mathbb{N}$, there exists $\phi_k \in \mathcal{N}\mathcal{N}(k,1,3)$, such that 
$\phi_k(x)=0$ for $x\le0$; $\phi_k(x)\in[0,1]$ for $x\in[0,1]$; and 
$\lvert x^2-\phi_k(x)\rvert\le \tfrac{1}{2k^2}$ for $x\in[0,1]$.
\end{lemma}

Our proof follows a similar strategy to \citet{jiao2023approximation}, who study approximation rates under the operator norm $\|A\| := \sup_{\|x\|_\infty \le 1} \|Ax\|_\infty$. Although both approaches construct shallow, wide neural networks based on Riemann sum approximations and achieve the same norm bound $3$, our construction relies on the Frobenius norm, which is more aligned with practical implementations. Compared with the approximation results in \citep{yarotsky2017error}, the key differences of our result are as follows.
First, Yarotsky constructs an approximation to $x^2$ using a deep ReLU network via composition of sawtooth-type functions, whereas our Lemma~\ref{lem:approx_x2} uses a shallow but wide construction. 
Second, Lemma~\ref{lem:approx_x2} achieves the approximation rate
$\varepsilon \asymp k^{-2}$ with a uniform Frobenius norm bound $K \le 3$ independent of the width $k$. 
In contrast, in Yarotsky's deep compositional construction, 
the accuracy satisfies $\varepsilon \asymp 2^{-2L}$ for depth $L$, and the resulting multiplicative Frobenius constraint satisfies 
\(
K = \prod_{\ell=1}^L(2^{-2\ell} \sqrt{2^2+(-4)^2+2^2}\sqrt{3+\tfrac54+1}))\lesssim 2^{-L(L+1)/2} 12^L.
\)
Equivalently, $K \asymp \exp\!\big(-c(\log(1/\varepsilon))^2\big)$ for some constant $c>0$. 
Thus, the approximation accuracy and the norm budget are intrinsically coupled in a sub-exponential manner.


\begin{proof}[Proof of Lemma~\ref{lem:approx_x2}]
For $x\in[0,1]$, note that
\[
x^2 = \int_0^x 2(x-b) db = \int_0^1 2 \sigma(x-b) db,
\]
where $\sigma(t)=\max\{0,t\}$ denotes the ReLU activation.
Fix $k\in\mathbb{N}$ and define
\begin{equation}\label{eq:phi-k-square}
\phi_k(x)
:= \frac{1}{k}\sum_{i=1}^k 2 \sigma \left(x-\frac{2i-1}{2k}\right),
\quad x\in\mathbb{R}.
\end{equation}
The function $\phi_k$ is the midpoint Riemann sum approximation of the above integral
using $k$ equally spaced points in $[0,1]$.
Clearly, $\phi_k(x)=0$ for $x\le 0$, and $\phi_k(x)\in[0,1]$ for all $x\in[0,1]$.
Following the proof of \citep[Lemma 3.3]{jiao2023approximation}, the approximation error can be bounded by
\[
\sup_{x\in[0,1]} |x^2-\phi_k(x)|
\le \frac{1}{2k^2}.
\]
The function $\phi_k$ can be realized by a ReLU network with one hidden layer and width
$k$. 
The corresponding parameters are
\[
A_0 = (1,\ldots,1)^\top \in \mathbb{R}^{k\times 1},
\quad
b_0 = \left(-\frac{2i-1}{2k}\right)_{i=1}^k \in \mathbb{R}^k,
\quad
A_1 = \left(\frac{2}{k},\ldots,\frac{2}{k}\right)\in\mathbb{R}^{1\times k},
\]
The above construction satisfies
\[
\| A_1\|_F = \frac{2}{\sqrt{k}},
\]
and since $\sum_{i=1}^k (2i-1)^2=\sum_{i=1}^k 4i^2+\sum_{i=1}^k 1- \sum_{i=1}^k 4i= 4\cdot \frac{k(k+1)(2k+1)}{6}
- 4\cdot \frac{k(k+1)}{2}
+ k = \frac{4k^3 - k}{3},$ there holds
\begin{equation*}\label{eq:A0b0-norm}
\begin{aligned}
\|A_0\|_F^2 + \|b_0\|_2^2+1 = k + \frac{1}{4k^2}\sum_{i=1}^k (2i-1)^2 +1
= \frac{4k}{3} - \frac{1}{12k}+1 ,
\end{aligned}
\end{equation*}
Therefore,
\[
\kappa(\theta) = \frac{2}{\sqrt{k}} \sqrt{\frac{16k^2+12k-1}{12k}}.
\]
Since this quantity is strictly decreasing for all $k \ge 1$, it follows that
$\kappa(\theta) \le 3$. 
Hence, one may choose $K = 3$ independently of the width $k$, which completes the proof.
\end{proof}

Note that multiplication can be expressed as a linear combination of square functions, namely
$xy = 2\big(((x+y)/2)^2 - (x/2)^2 - (y/2)^2\big).$
The scaling of $x+y$ ensures that each argument lies in the interval $[-1,1]$.
We then approximate $xy$ by applying the linear combination property (Proposition~\ref{lem:linearcmb}) together with Lemma~\ref{lem:approx_x2}.

\begin{lemma}[Approximation for $xy$]\label{lem:approx_xy}
    For any $k \in \mathbb{N}$, there exists a network 
$\psi_k \in \mathcal{N}\mathcal{N}(6k,2,360)$ with 
$\psi_k : [-1,1]^2 \to [-1,1]$ such that 
\[
|xy - \psi_k(x,y)| \le \frac{3}{k^2}, \quad \forall  (x,y) \in [-1,1]^2,
\]
and $\psi_k(x,y)=0$ whenever $xy=0$. 
\end{lemma}

\begin{proof}[Proof of Lemma~\ref{lem:approx_xy}]
Let $\phi_k\in  \mathcal{N}\mathcal{N}(k,1,3)$ be the one-hidden-layer network from Lemma~\ref{lem:approx_x2},
which satisfies $\phi_k(t)\in[0,1]$ for $t\in[0,1]$, $\phi_k(t)=0$ for $t \leq 0$ and
\[
|t^2-\phi_k(t)|\le \frac{1}{2k^2},\quad t\in[0,1].
\]
Define the even extension $\widetilde\phi_k(t):=\phi_k(t)+\phi_k(-t)$ for $t \in [-1,1]$ .
Then $\widetilde\phi_k\in  \mathcal{N}\mathcal{N}(2k,1,6)$, and for all
$t\in[-1,1]$ we have
\[
|t^2-\widetilde\phi_k(t)|
=
\left|t^2-\phi_k(|t|)\right|
\le \frac{1}{2k^2},
\quad
\widetilde\phi_k(t)\in[0,1].
\]

({Network structure}).
Based on the observation $xy = 2\big(((x+y)/2)^2 - (x/2)^2 - (y/2)^2\big),$   we construct three networks on $[-1,1]^2$,
\[
\Psi_1(x,y):=\widetilde\phi_k \left(\frac{x+y}{2}\right),\quad
\Psi_2(x,y):=\widetilde\phi_k \left(\frac{x}{2}\right),\quad
\Psi_3(x,y):=\widetilde\phi_k \left(\frac{y}{2}\right).
\]
Take $\Psi_1(x,y)$ as an example, if $\tilde{\phi}_k(x) = A_1 \sigma(A_0x+b_0)$, $A_0 \in \R^{2k \times 1}$, then \[\Psi_1(x,y) = A_1\sigma \left(A_0 \begin{pmatrix}
  \frac{1}{2}& \frac{1}{2}  
\end{pmatrix} 
\begin{pmatrix}
    x\\
    y
\end{pmatrix}+b_0\right), \] 
then the new $A_0':= A_0 \left(1/2 \; 1/2 \right)\in \R^{2k\times2}$
and $\Psi_1 \in  \mathcal{N}\mathcal{N}(2k, 1, K_{\Psi_1})$, where the norm is bounded by 
\[K_{\Psi_1} \leq \|A_1\|_F \sqrt{ \frac{1}{2}\|A_0\|_F^2+\|b_0\|_F^2+1}\le K_{\tilde{\phi}_k}\leq 6.\]
Similarly, we have $\Psi_2 \in  \mathcal{N}\mathcal{N}(2k,1,6)$ and $\Psi_3 \in  \mathcal{N}\mathcal{N}(2k,1,6)$.
Then, set
\begin{equation}\label{eq:psi-k}
\Psi(x,y)
:=
2\Big(\Psi_1(x,y)-\Psi_2(x,y)-\Psi_3(x,y)\Big).
\end{equation}
By Proposition~\ref{lem:linearcmb} (linear combination),  its norm constraint
is bounded by 
\[K_{\Psi} \le\sqrt{4}\sqrt{3\times (2\times6)^2}\le48 \]
and $\Psi \in  \mathcal{N}\mathcal{N}(6k,1,36)$. To make the output in the range of $[-1,1]$, we need an additional layer $\chi(x)$ given by
\[
\chi(x)
=
\sigma(x) - \sigma(-x)
- \sigma \left(x - 1\right)
+ \sigma \left(-x - 1\right)
= (x \vee (-1)) \wedge 1.
\]
Here, \(a \vee b := \max\{a,b\}\) and \(a \wedge b := \min\{a,b\}\). Its norm is bounded by
\[K_\chi = \sqrt{1+1+1+1}\sqrt{1+1+1+1+1 +1+1}=2\sqrt{7}.\]
That is, $\chi \in \CN\CN(4,1,2\sqrt{7})$.
Finally, define $\psi_k(x,y) = \chi \circ \Psi(x,y)$. By the compositional property (Proposition~\ref{lem:comp}), we have
\[K_{\psi_k} \leq \sqrt{2}\cdot 2\sqrt{7}\cdot \sqrt{48^2+2}\approx 359.4\leq 360\]

(Approximation error).
For $(x,y)\in[-1,1]^2$, letting $a=\frac{x+y}{2}$, $b=\frac{x}{2}$, $c=\frac{y}{2}$,
we have $a,b,c\in[-1,1]$ and
\begin{align*}
|xy-\psi_k(x,y)|
&\le
2\Big(
|a^2-\widetilde\phi_k(a)|
+
|b^2-\widetilde\phi_k(b)|
+
|c^2-\widetilde\phi_k(c)|
\Big)
\le
2\cdot 3\cdot \frac{1}{2k^2}
=
\frac{3}{k^2}.
\end{align*}

(Vanishing on the coordinate axes).
If $xy=0$, then either $x=0$ or $y=0$, and hence $a=\pm c$ or $a=\pm b$.
Since $\widetilde\phi_k$ is even, it follows from \eqref{eq:psi-k} that
$\psi_k(x,y)=0$ whenever $xy=0$.
This completes the proof.   
\end{proof}

After establishing the multiplication operator, the next key step is to extend it to $d$-degree monomials.
\begin{lemma}[Approximation of $d$-degree monomials]\label{lem:approx_monomials}
    For any $d \ge 2$ and $k \in \mathbb{N}$, there exists 
$\phi \in  \mathcal{N}\mathcal{N}(6dk, 2\lceil \log_2 d \rceil , K)$  with $K \le  722^{\lceil \log_2 d  \rceil} \cdot {2}^{\frac{7\lceil \log_2 d  \rceil^2-7\lceil \log_2 d  \rceil}{4}} $
such that $\phi : [-1,1]^d \to [-1,1]$, $\phi(x) = 0$ if any $x_i = 0$ and 
\[
|x_1 \cdots x_d - \phi(x)| \le \frac{6d}{k^2}, 
\quad x = (x_1,\ldots,x_d)^\top \in [-1,1]^d.
\] 
\end{lemma}

\begin{proof}[Proof of Lemma~\ref{lem:approx_monomials}]The proof of this lemma relies on the multiplication operator constructed inductively, 
as well as on the concatenation and combination properties of neural networks.
Let $s:=\lceil \log_2 d\rceil$ and set $m:=2^s$, so that $d\le m <2d$.
Extend $x=(x_1,\dots,x_d)\in[-1,1]^d$ to $\bar x\in[-1,1]^m$ by padding with ones:
\[
\bar x_i=x_i,\quad i=1,\dots,d,
\quad
\bar x_i=1,\quad i=d+1,\dots,m.
\]
Then $\prod_{i=1}^m \bar x_i=\prod_{i=1}^d x_i$.









\noindent{(Binary-tree network construction).}
Let $\psi_k\in  \mathcal{N}\mathcal{N}(6k,2,360)$ be the bivariate multiplier from Lemma~\ref{lem:approx_xy},
so that $\psi_k:[-1,1]^2\to[-1,1]$,
\[
|uv-\psi_k(u,v)|\le \frac{3}{k^2},\quad \forall (u,v)\in[-1,1]^2,
\quad
\psi_k(u,v)=0\ \text{whenever }uv=0.
\]
We construct the network as a binary tree employing $\psi_k$ inductively. For each tree level,  define vectors $z^{(\ell)}\in[-1,1]^{m/2^\ell}$ by
\[
z^{(0)}:=\bar x, 
\quad
z^{(\ell)}_j := \psi_k \left(z^{(\ell-1)}_{2j-1}, z^{(\ell-1)}_{2j}\right), \quad \ell=1,\ldots, s,
\quad j=1,\dots,\frac{m}{2^{\ell}}.
\]
And denote the composition of all the outputs of the $\ell$-th level as
\[\Phi_{\ell} = \left(z_1^{(\ell)}, \ldots, z_{m/2^\ell}^{(\ell)}\right)\]
Set the final output as $\phi(x):=z^{(s)}_1$. By construction, $\phi(x)=0$ whenever $x_i=0$ for some $i$,
since at the first level the corresponding pair-product is zero, and this property
propagates through all subsequent applications of $\psi_k$.

\begin{figure}[htbp]
\centering
\begin{tikzpicture}[
    scale=0.6,
    transform shape,
    leaf/.style={circle, draw=red!70, fill=red!12, thick, minimum size=7mm},
    internal/.style={circle, draw=blue!70, fill=blue!12, thick, minimum size=7mm},
    pad/.style={circle, draw=gray!70, fill=gray!20, thick, minimum size=7mm},
    root/.style={circle, draw=green!50!black, fill=green!12, thick, minimum size=8mm},
    arrow/.style={-Stealth, thick},
    every node/.style={font=\small}
]

\node[leaf] (x1) at (0,0) {$x_1$};
\node[leaf] (x2) at (1.4,0) {$x_2$};
\node[leaf] (x3) at (2.8,0) {$x_3$};
\node[leaf] (x4) at (4.2,0) {$x_4$};
\node[leaf] (x5) at (5.6,0) {$x_5$};
\node[pad] (x6) at (7.0,0) {$1$};
\node[pad] (x7) at (8.4,0) {$1$};
\node[pad] (x8) at (9.8,0) {$1$};

\node[below=2pt of x1, draw=none] {$z_1^{(0)}$};
\node[below=2pt of x2, draw=none] {$z_2^{(0)}$};
\node[below=2pt of x3, draw=none] {$z_3^{(0)}$};
\node[below=2pt of x4, draw=none] {$z_4^{(0)}$};
\node[below=2pt of x5, draw=none] {$z_5^{(0)}$};
\node[below=2pt of x6, draw=none] {$z_6^{(0)}$};
\node[below=2pt of x7, draw=none] {$z_7^{(0)}$};
\node[below=2pt of x8, draw=none] {$z_8^{(0)}$};

\node[internal] (n11) at (0.7,1.6) {$\psi_k$};
\node[internal] (n12) at (3.5,1.6) {$\psi_k$};
\node[internal] (n13) at (6.3,1.6) {$\psi_k$};
\node[internal] (n14) at (9.1,1.6) {$\psi_k$};

\node[above=2pt of n11, draw=none] {$z_1^{(1)}$};
\node[above=2pt of n12, draw=none] {$z_2^{(1)}$};
\node[above=2pt of n13, draw=none] {$z_3^{(1)}$};
\node[above=2pt of n14, draw=none] {$z_4^{(1)}$};

\node[internal] (n21) at (2.1,3.2) {$\psi_k$};
\node[internal] (n22) at (7.7,3.2) {$\psi_k$};

\node[above=2pt of n21, draw=none] {$z_1^{(2)}$};
\node[above=2pt of n22, draw=none] {$z_2^{(2)}$};

\node[root] (root) at (4.9,4.8) {$\psi_k$};
\node[above=2pt of root, draw=none] {$z_1^{(3)}=\phi(x)$};

\draw[arrow] (x1) -- (n11);
\draw[arrow] (x2) -- (n11);
\draw[arrow] (x3) -- (n12);
\draw[arrow] (x4) -- (n12);
\draw[arrow] (x5) -- (n13);
\draw[arrow] (x6) -- (n13);
\draw[arrow] (x7) -- (n14);
\draw[arrow] (x8) -- (n14);

\draw[arrow] (n11) -- (n21);
\draw[arrow] (n12) -- (n21);
\draw[arrow] (n13) -- (n22);
\draw[arrow] (n14) -- (n22);

\draw[arrow] (n21) -- (root);
\draw[arrow] (n22) -- (root);

\node[left] at (-0.8,0) {$\ell=0$};
\node[left] at (-0.8,1.6) {$\ell=1$};
\node[left] at (-0.8,3.2) {$\ell=2$};
\node[left] at (-0.8,4.8) {$\ell=3$};

\end{tikzpicture}
\caption{Binary-tree construction of the monomial approximator, illustrated for the case $m=8$ and $s=3$. The input is padded with ones so that the number of leaves is $m=2^s$. Each internal node applies the bivariate multiplier $\psi_k$ to a pair of children, producing the next-level variables $z_j^{(\ell)}$. After $s$ levels, the root outputs $\phi(x)=z_1^{(s)}$, which approximates $\prod_{i=1}^d x_i$.}
\label{fig:monomial_tree}
\end{figure}
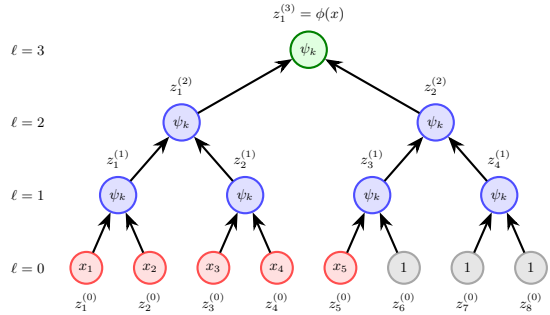

Observe that the width doubles when two copies of $\psi_k$ are evaluated in parallel; hence
the widest width happens when $\ell=1$, that is, $W = 6k\cdot \frac{m}{2}\leq6kd$.
Moreover, the depth increases by $2$ for each level, so that
the depth will be $2 s = 2 \lceil \log_2 d \rceil$.
 By the concatenation Proposition~\ref{lem:concate},
these $m/2^\ell$  parallelization yield a depth-$2$ network $\Phi_\ell=(\psi_k,\ldots,\psi_k)$ whose
norm satisfies
\[
K(\Phi_\ell)\le \left(\sqrt{\frac{m}{2^\ell}+1}\right)^2\sqrt{\frac{m}{2^\ell} \cdot 360^2} \le720\left(\frac{m}{2^{\ell}}\right)^{3/2}.
\]
The full network $\phi=\Phi_{s}\circ\cdots\circ\Phi_1$ is obtained
by composing these $s$ levels and each $\Phi_i \in \CN\CN(6kd,2,720\left({m}/{2^{i}}\right)^{3/2})$. Using the composition proposition~\eqref{lem:comp}, we obtain
\[
\begin{aligned}
K(\phi)
&\le
(\sqrt{2})^{(s-1)2s}
\prod_{\ell=1}^{s}\left(K(\Phi_\ell)+2\right) \\
&\le
2^{s(s-1)}
\prod_{\ell=1}^{s}
722\left(2^{s-\ell}\right)^{3/2} =
722^s
2^{\frac{7}{4}s(s-1)}.
\end{aligned}
\]
Therefore, the generated neural network $\phi \in \CN\CN(6kd,2s,722^s
2^{\frac{7}{4}s(s-1)})$.

\noindent{(Error bound).}
Let $P^{(\ell)}_j := \prod_{i=1}^{2^\ell}\bar x_{(j-1)2^\ell+i}$ denote the exact block products,
and define the uniform error at level $\ell$ by
\[
e_\ell := \max_{1\le j\le m/2^\ell}  \left|z^{(\ell)}_j - P^{(\ell)}_j\right|.
\]
Clearly $e_0=0$. Since all quantities lie in $[-1,1]$, for each step we have
\begin{align*}
&\left|z^{(\ell+1)}_j - P^{(\ell+1)}_j\right|=
\left|\psi_k\big(z^{(\ell)}_{2j-1}, z^{(\ell)}_{2j}\big)
- P^{(\ell)}_{2j-1}P^{(\ell)}_{2j}\right| \\
&\le \left|\psi_k\big(z^{(\ell)}_{2j-1}, z^{(\ell)}_{2j}\big)
- z^{(\ell)}_{2j-1} z^{(\ell)}_{2j}\right|  + \left|z^{(\ell)}_{2j-1} z^{(\ell)}_{2j}
- P^{(\ell)}_{2j-1}P^{(\ell)}_{2j}\right| \\
&\le \frac{3}{k^2}
+ \left|z^{(\ell)}_{2j-1}-P^{(\ell)}_{2j-1}\right|
+ \left|z^{(\ell)}_{2j}-P^{(\ell)}_{2j}\right| \le \frac{3}{k^2} + 2e_\ell,
\end{align*}
where we used $|ab-a'b'|\le |a-a'|+|b-b'|$ for $a,b,a',b'\in[-1,1]$.
Taking the maximum over $j$ yields
\[
e_{\ell+1}\le 2e_\ell + \frac{3}{k^2}.
\]
Solving it with $e_0=0$ gives
\[
e_s \le \frac{3}{k^2}\sum_{t=0}^{s-1}2^t = \frac{3(2^s-1)}{k^2}
= \frac{3(m-1)}{k^2}
\le \frac{3m}{k^2}
\le \frac{6d}{k^2},
\]
since $m<2d$. Noting that $P^{(s)}_1=\prod_{i=1}^m\bar x_i=\prod_{i=1}^d x_i$, we obtain
\[
\left|x_1\cdots x_d-\phi(x)\right|
=|P^{(s)}_1-z^{(s)}_1|
\le e_s
\le \frac{6d}{k^2},
\quad x\in[-1,1]^d.
\]
This completes the proof.
\end{proof}

We are now ready to prove the approximation result for H\"older spaces.

\begin{proof}[Proof of \thmref{thm:apx_s_F}] The proof is divided into three steps.

\textbf{Step 1: Approximation based on local Taylor expansion $p(\bx)$. }
First, define \[\psi(t) = \sigma(1 - |t|)=\sigma\big(1 - \sigma(t) - \sigma(-t)\big) \in[0,1].\]
This function is a ``hat function" supported on $[-1,1]$ 
and can be realized by a small ReLU network with norm bounded by $1\cdot\sqrt{1+1+1+1} \sqrt{1+1+1}=2\sqrt{3}$. That is,
$\psi \in \mathcal{N}\mathcal{N}(2,2,2\sqrt{3})$.
For each multi-index $\mathbf n = (n_1,\ldots,n_d) \in \{0,1,\ldots,N\}^d$, we define the basis functions by scaling and shifting $\psi$ as
\[
\psi_{\mathbf n}(\bx) := \prod_{i=1}^d \psi(Nx_i - n_i), \quad \bx \in [0,1]^d.
\]
From this construction, for each fixed $n_i$, the factor $\psi(Nx_i-n_i)$ is nonzero only when 
\(
x_i \in \left[\frac{n_i-1}{N},\, \frac{n_i+1}{N}\right].
\)
Hence, each $\psi_{\mathbf n}$ has compact support localized around the grid point $\mathbf n/N$, and for any $\bx \in [0,1]^d$, at most $2^d$ such functions are active.
Moreover, these functions form a partition of unity:
\[
\sum_{\mathbf n} \psi_{\mathbf n}(\bx) = 1, \quad \forall \bx \in [0,1]^d,
\]
which follows from the corresponding one-dimensional property of $\psi$ and the tensor-product structure.
This locality leads to representations that depend only on nearby regions of the domain, which is crucial for stable approximation and error control. Figure~\ref{fig:partition} illustrates this construction. 
\begin{figure}[htbp]
    \centering
    \includegraphics[width=0.6\linewidth]{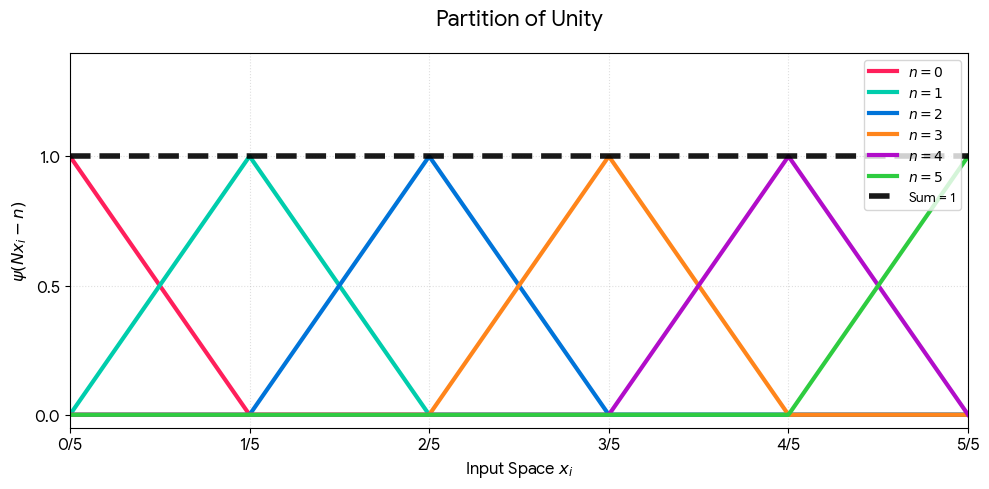}
    \caption{{Partition of Unity via Localized Hat Functions.} 
    The plot illustrates the basis functions $\psi(Nx_i - n_i)$ for $N=5$ and indices $n_i \in \{0, \dots, 5\}$. 
    Each colored triangle represents a localized basis function centered at $x_i = n_i/N$ with support on 
    $[\frac{n_i-1}{N}, \frac{n_i+1}{N}]$. The dashed black line indicates the sum $\sum_{n} \psi_n(x)$, 
    which is identically equal to $1$ across the domain $[0, 1]$.}
    \label{fig:partition}
\end{figure}

For any $h\in \CC^\alpha([0,1]^d)$, set its local Taylor expansion
\[
p(\mathbf x) = \sum_{\mathbf n}\sum_{\|\mathbf s\|_1 \le r} c_{{\mathbf n},{\mathbf s}}p_{\mathbf n, \mathbf s}(\mathbf x), \quad p_{\mathbf n, \mathbf s}(x) =  \psi_{\mathbf n}(\mathbf x) \left( \mathbf x - \frac{{\mathbf n}}{N} \right)^{\mathbf s},
\]
where $c_{\mathbf n, \mathbf s} = \dfrac{\partial^{\mathbf s}h \left(\tfrac{\mathbf n}{N}\right)}{\mathbf s!} \le \mathbf 1$ and $p_{\mathbf n, \mathbf s}$ is supported on $\{\mathbf x\in \R^d: \|\mathbf x- \frac{\mathbf n}{N}\|_\infty \leq \frac{1}{N} \}$. Then
\begin{align*}
|h(\mathbf x)-p(\mathbf x)|
&\le \sum_{\mathbf n} \psi_{\mathbf n}(\mathbf x)
\Biggl|
h(\mathbf x)-\sum_{\|\mathbf s\|_1\le r} c_{\mathbf n,\mathbf s}\left(\mathbf x-\frac{\mathbf n}{N}\right)^{\mathbf s}
\Biggr| \\[2pt]
&\le
\sum_{\mathbf n: \left\|\mathbf x-\frac{\mathbf n}{N}\right\|_\infty\le \frac{1}{N}}
\psi_{\mathbf n}(\mathbf x)
\left|
h(\mathbf x) -\sum_{\|\mathbf s\|_1\le r}
c_{\mathbf n,\mathbf s}
\left(\mathbf x-\frac{\mathbf n}{N}\right)^{\mathbf s}
\right|\\
&\le
\sum_{\mathbf n: \left\|\mathbf x-\frac{\mathbf n}{N}\right\|_\infty\le \frac{1}{N}}
d^{ r} 
\left\|\mathbf x-\frac{\mathbf n}{N}\right\|_\infty^{\alpha} \leq 2^d d^r N^{-\alpha},
\end{align*}
where we used the Taylor remainder estimate with 
$\boldsymbol{\xi}$ lying on the line segment between 
$\mathbf x$ and $\frac{\mathbf n}{N}$g
\begin{align*}
\left| h(\mathbf{x}) - \sum_{\|\mathbf{s}\|_1 \le r} c_{\mathbf{n},\mathbf{s}} \left(\mathbf{x} - \frac{\mathbf{n}}{N}\right)^{\mathbf{s}} \right|
&= \left| \sum_{\|\mathbf{s}\|_1 = r} \frac{1}{\mathbf{s}!} \left( \partial^{\mathbf{s}} h(\boldsymbol{\xi}) - \partial^{\mathbf{s}} h\!\left(\frac{\mathbf{n}}{N}\right) \right) \left(\mathbf{x} - \frac{\mathbf{n}}{N}\right)^{\mathbf{s}} \right| \\
&\le \sum_{\|\mathbf{s}\|_1 = r} \frac{1}{\mathbf{s}!} \left| \partial^{\mathbf{s}} h(\boldsymbol{\xi}) - \partial^{\mathbf{s}} h\!\left(\frac{\mathbf{n}}{N}\right) \right| \cdot \left| \left(\mathbf{x} - \frac{\mathbf{n}}{N}\right)^{\mathbf{s}} \right| \\
&\le \sum_{\|\mathbf{s}\|_1 = r} \frac{1}{\mathbf{s}!} \left\| \boldsymbol{\xi} - \frac{\mathbf{n}}{N} \right\|_\infty^{\alpha - r}
\cdot \left\| \mathbf{x} - \frac{\mathbf{n}}{N} \right\|_\infty^{r} \\
&\le \left( \sum_{\|\mathbf{s}\|_1 = r} \frac{1}{\mathbf{s}!} \right)
\left\| \mathbf{x} - \frac{\mathbf{n}}{N} \right\|_\infty^{\alpha} \le d^r \left\| \mathbf{x} - \frac{\mathbf{n}}{N} \right\|_\infty^{\alpha}
\end{align*}
and at most $2^d$ functions $\psi_{\mathbf n}$ overlap at any point $\bx$  for the last step. 

\textbf{Step 2: Approximate $p(\bx)$ by neural networks.}
 Note that there are 
$ d + \|\mathbf s\|_1 \le t$ multiplier factors in $p_{\mathbf n, \mathbf s}(\bx)$, where $t=d+r$.
By Lemma~\ref{lem:approx_monomials}, for any $k \in \mathbb{N}$, 
there exists a network $\phi_t \in \mathcal{N}\mathcal{N}(6tk, 2\lceil \log_2 t \rceil, K_t)$ 
that approximates the $t$-fold product on $[-1,1]^t$ with error
\begin{equation}\label{equ:t_fold}
  \left| \prod_{i=1}^t x_i - \phi_t(x_1,\ldots,x_t) \right| \le \frac{6t}{k^2},  
\end{equation}
and its norm $K_t$ satisfies
\begin{equation} \label{equ:D_monomial}
   K_t \le  722^{\lceil \log_2 t  \rceil} \cdot {2}^{\frac{7\lceil \log_2 t  \rceil^2-7\lceil \log_2 t  \rceil}{4}}.  
\end{equation}
Next, we construct a network for each $p_{\mathbf n, \mathbf{s}}$ and then do the linear combination of them. 
Slightly abusing the notation, we observe that
\[
\psi(Nx_i-n_i)=
\sigma(1-\sigma(Nx_i-n_i)-\sigma(-Nx_i+n_i))
\]
and $1\cdot\sqrt{1+1+1+1}\cdot \sqrt{2N^2+2n_i^2+1}\leq 2\sqrt{5}N\leq 2\sqrt{15}N$, therefore \\$\psi(Nx_i-n_i) \in \mathcal N \mathcal N(2,2,2\sqrt{15}N)$. 

In addition, it is straightforward to observe that
\begin{align*}
\left (x_i-\frac{n_i}{N} \right)
&= \sigma\left (
x_i-\frac{n_i}{N}
 \right)-
 \sigma\left (
-x_i+\frac{n_i}{N}
 \right)
 \\
 &=
 \sigma\left (
\sigma\left (
x_i-\frac{n_i}{N}
 \right) \right ) -
 \sigma \left (\sigma\left (
-x_i+\frac{n_i}{N}
 \right)
 \right)
\end{align*}
and thus $2\sqrt{1+1+1}\cdot \sqrt{1+1+2\frac{n_i^2}{N^2}+1}\leq 2\sqrt{15}\leq 2\sqrt{15}N$. This implies
$ \left (x_i-\frac{n_i}{N} \right) \in \mathcal N \mathcal N(2,2,2\sqrt{3}\sqrt{5})$. By  Proposition \ref{lem:concate}, it follows that
\[
\left (\psi(Nx_1-n_1), \ldots, \psi(Nx_d-n_d), \ldots, x_i-\frac{n_i}{N},\ldots \right )
\in \mathcal N \mathcal N
(2(d+r),2,2\sqrt{15}(d+r)\sqrt{d+r}N).
\]
Then, by \eqref{equ:t_fold} and by Proposition \ref{lem:comp} we obtain a network
\[\phi_{\bn,\bs}=\phi_t(\psi(Nx_1-n_1), \ldots, \psi(Nx_d-n_d), \ldots, x_i-n_i,\ldots )\] 
such that $W_{\phi_{\bn,\bs}}=12(d+r)k, D_{\phi_{\bn,\bs}}=2\lceil \log_2(d+r)\rceil+2$  and
\[
\left|\phi_{\bn,\bs}(\bx)-p_{\bn,\bs}(\bx) \right|\leq \frac{6(d+r)}{k^2}.
\]
For its norm bound, since 
$2\sqrt{15}(d+r+1)\sqrt{d+r}N)K_{\phi_t}\leq 8(d+r)^2N K_{\phi_t}$, we have
\[
K_{\phi_{\bn,\bs}}\leq 
8(d+r)^2N  \cdot 722^{\lceil \log_2 (2(d+r))  \rceil} \cdot {2}^{\frac{7\lceil \log_2 (2(d+r))  \rceil^2-7\lceil \log_2 (2(d+r))  \rceil}{4}}.
\]

Moreover, since $\phi_{\bn,\bs}(x_1, \ldots,x_t)=0$ when $x_1 \cdot x_2 \cdots x_t=0$,  $\phi_{\bn,\bs}$ is supported on \\$\{\mathbf x\in \R^d: \|\mathbf x- \frac{\mathbf n}{N}\|_\infty \leq \frac{1}{N} \}$. 

Define the following linear combination accordingly to approximate $p(\bx)$
\[\phi(\bx) = \sum_{\bn}\sum_{\|\bs\|_1\leq r}c_{\bn,\bs}\phi_{\bn,\bs}(\bx).\]
There are less than $(N+1)^d d^r$ terms in the summation.
By the linear combination (Proposition~\ref{lem:linearcmb}) and \eqref{equ:D_monomial}, we have
\[\phi(\bx)\in  \mathcal{N}\mathcal{N}(12(N+1)^d d^r(d+r)k, 2\lceil \log_2(d+r)\rceil+2 ,K_\phi)\] 
and the norm constraint $K_\phi$ can be bounded, said $D_\phi= 2\lceil \log_2(d+r)\rceil+2 $, by
\[K_\phi\le 
\left((N+1)^dd^r+1\right)^{\frac{D_\phi+1}{2}}8(d+r)^2N  \cdot 
722^{\lceil \log_2 (2(d+r))  \rceil} \cdot {2}^{\frac{7\lceil \log_2 (2(d+r))  \rceil^2-7\lceil \log_2 (2(d+r))  \rceil}{4}}. \]
The approximation error
\begin{align*}
|p(\mathbf{x}) - \phi(\mathbf{x})\vert{} 
&= \left\vert{} \sum_{\mathbf{n}} \sum_{\Vert{}\mathbf{s}\Vert{}_1 \leq r} c_{\mathbf{n}, \mathbf{s}} p_{\mathbf{n}, \mathbf{s}}(\mathbf{x}) - \sum_{\mathbf{n}} \sum_{\Vert{}\mathbf{s}\Vert{}_1 \leq r} c_{\mathbf{n}, \mathbf{s}} \phi_{\mathbf{n}, \mathbf{s}}(\mathbf{x}) \right\vert{} \\
&\leq \sum_{\mathbf n: \left\|\mathbf x-\frac{\mathbf n}{N}\right\|_\infty\le \frac{1}{N}}\sum_{\Vert{}\mathbf{s}\Vert{}_1 \leq r} \vert{}c_{\mathbf{n}, \mathbf{s}}\vert{} \cdot \left\vert{} p_{\mathbf{n}, \mathbf{s}}(\mathbf{x}) - \phi_{\mathbf{n}, \mathbf{s}}(\mathbf{x}) \right\vert{} \\ 
&\leq \sum_{\mathbf n: \left\|\mathbf x-\frac{\mathbf n}{N}\right\|_\infty\le \frac{1}{N}}\sum_{\Vert{}\mathbf{s}\Vert{}_1 \leq r} 1 \cdot \frac{6(d+r)}{k^2} \\ 
&=  2^d d^r \cdot \frac{6(d+r)}{k^2} 
\end{align*}
where we used that $2^d$ functions $\psi_{\mathbf n}$ overlap at any point $\bx$  for the last step

\textbf{Step 3: Combining error bounds.}
The overall approximation error is bounded by
\begin{align}\label{equ:apx_s}
&\|h-\phi\|_{L^\infty([0,1]^d)} \notag \\
& \le\|h-p\|_{L^\infty([0,1]^d)} + \|p-\phi\|_{L^\infty([0,1]^d)} \notag\\
&\le 2^{d} d^{r} N^{-\alpha}
+
\frac{6 \cdot 2^{d} (d+r)d^{r}}{k^{2}}.
\end{align}
It gives a rate of order $k^{-2}$ by choosing $N \asymp \lceil k^{2/\alpha} \rceil$. Then the associated network has width $W \asymp k^{2d/\alpha+1} $, depth $D=2\lceil \log_2(d+r)\rceil+2$ and norm constraint $K\asymp  N^{d(D+1)/2+1} \asymp k^{(d(D+1)+2)/\alpha}$. Since the approximation rate in \eqref{equ:apx_s} decreases with increasing $N$ and $k$ (increasing width $W$) for a fixed depth, it follows that
\[
\|h - \phi\|_{L^\infty([0,1]^d)} 
\lesssim K^{-\frac{2\alpha}{2+d(D+1)}}
\]
provided that $W \gtrsim K^{\frac{2d+\alpha}{d(D+1)+2}}$. 
This completes the proof.
\end{proof}

\subsection{Proof of approximating sparse compositional functions}\label{subsec:proof_apx_cmp}
Before presenting the proof of Theorem~\ref{thm:apx_cmp}, we show that the result extends to intermediate functions of the form 
\[
g_{v^{(\ell)}} \in C^{\alpha_{v^{(\ell)}}}
\Big([-R_{v^{(\ell-1)}},R_{v^{(\ell-1)}}]^{d_{\text{in}(v^{(\ell)})}}\Big),
\quad 
\|g_{v^{(\ell)}}\|_\infty \le R_{v^{(\ell)}},
\]
since they can be reduced to the unit-cube setting via a simple rescaling.
For $\ell = 1, \ldots, L-1$, define
\[
h_{\mathcal{V}^{(\ell)}}(x)
:=
\frac{ g_{\mathcal{V}^{(\ell)}}\!\big( 2R_{\mathcal{V}^{(\ell-1)}} x - R_{\mathcal{V}^{(\ell-1)}} \big) }{2R_{\mathcal{V}^{(\ell)}}}
+ \frac{1}{2},
\]
where the affine transformation $2R_{\mathcal{V}^{(\ell-1)}} x - R_{\mathcal{V}^{(\ell-1)}}$ is applied coordinate-wise, i.e., each parent variable $x_u$ is mapped to $2R_u x_u - R_u$.
At the boundary levels, define
\[
h_{\mathcal{V}^{(0)}} := \frac{g_{\mathcal{V}^{(0)}}}{2R_0} + \frac{1}{2},
\qquad
h_{\mathcal{V}^{(L)}}(x) := g_{\mathcal{V}^{(L)}}\!\big( 2R_{\mathcal{V}^{(L-1)}} x - R_{\mathcal{V}^{(L-1)}} \big).
\]

Then, for $\ell = 1, \ldots, L-1$, we have
\[
h_{v^{(\ell)}} \in C^{\alpha_{v^{(\ell)}}}\big([0,1]^{d_{\text{in}(v^{(\ell)})}}, 1\big), \quad
\text{and} \quad
h_{v^{(L)}} \in C^{\alpha_{v^{(L)}}}\big([0,1]^{|\mathrm{pa}(v^{(L)})|}, R_{v^{(L)}}\big).
\]
Consequently, the composition $f$ can be equivalently written as
\[
f = g_{\mathcal{V}^{(L)}} \circ \cdots \circ g_{\mathcal{V}^{(0)}}
  = h_{\mathcal{V}^{(L)}} \circ \cdots \circ h_{\mathcal{V}^{(0)}}.
\]

\begin{proof}[Proof of Theorem~\ref{thm:apx_cmp}]
 The proof contains three parts.

\textbf{(Networks construction).} 
By \thmref{thm:apx_s_F}, each $g_{v^{(\ell)}}$ can be approximated by a neural network $\phi_{v^{(\ell)}} \in \CN\CN(W_{v^{(\ell)}},{D}_{v^{(\ell)}},K_{v^{(\ell)}})$ with

\begin{align}
    W_{v^{(\ell)}} &= 12({d_{\text{in}(v^{(\ell)})}}+{r_{v^{(\ell)}}}){d_{\text{in}(v^{(\ell)})}}^{r_{v^{(\ell)}}} (N_{v^{(\ell)}}+1)^{d_{\text{in}(v^{(\ell)})}} k_{v^{(\ell)}},\notag\\
   D_{v^{(\ell)}}  &= 2\left\lceil \log_2({d_{\text{in}(v^{(\ell)})}}+{r_{v^{(\ell)}}}) \right\rceil+2,\label{equ:block_depth}\\
   K_{v^{(\ell)}} &= 
\left((N_{v^{(\ell)}}+1)^{d_{\text{in}(v^{(\ell)})}}
d_{\text{in}(v^{(\ell)})}^{r_{v^{(\ell)}}}+1\right)^{\frac{D_{v^{(\ell)}}+1}{2}}8(d_{\text{in}(v^{(\ell)})}+r_{v^{(\ell)}})^{2} N_{v^{(\ell)}}\notag\\ 
&\cdot 
722^{\lceil \log_2 (2(d_{\text{in}(v^{(\ell)})}+r_{v^{(\ell)}}))  \rceil} \cdot {2}^{\frac{7\lceil \log_2 (2(d_{\text{in}(v^{(\ell)})}+r_{v^{(\ell)}}))  \rceil^2-7\lceil \log_2 (2(d_{\text{in}(v^{(\ell)})}+r_{v^{(\ell)}}))  \rceil}{4}}.\notag
\end{align}

and 
\begin{equation}\label{equ:apx_vl}
\begin{split}
    &\|g_{v^{(\ell)}} - \phi_{v^{(\ell)}}\|_{L^\infty([0,1]^{{d_{\text{in}(v^{(\ell)})}}})} \\
&\le 2^{d_{\text{in}(v^{(\ell)})}} {d_{\text{in}(v^{(\ell)})}}^{r_{v^{(\ell)}}} \left( N_{v^{(\ell)}}^{-\alpha_{v^{(\ell)}}} + \frac{6(d_{\text{in}(v^{(\ell)})}+r_{v^{(\ell)}}){d_{\text{in}(v^{(\ell)})}}^{r_{v^{(\ell)}}}}{k_{v^{(\ell)}}^2} \right)  
\end{split}  
\end{equation}
where the precise constants are given by the proof of \thmref{thm:apx_s_F}.
By the concatenation property (Proposition~\ref{lem:concate}) of neural networks, define
\[
\phi_{\CV^{(\ell)}} := \left( \phi_{v_1^{(\ell)}}, \ldots, \phi_{v^{(\ell)}_{|\CV^{(\ell)}|}}\right) \in \CN\CN\left(W_{\CV^{(\ell)}}, D_{\CV^{(\ell)}}, K_{\CV^{(\ell)}}\right).
\]
Here, the width $W_{\CV^{(\ell)}}= \sum_{v^{(\ell)}\in \CV^{(\ell)} } W_{v^{(\ell)}}$, the depth $D_{\CV^{(\ell)}} =\max_{v^{(\ell)}\in \CV^{(\ell)} }\left\{D_{v^{(\ell)}}\right\}$, and the norm bound $K_{\CV^{(\ell)}}= (|\CV^{(\ell)}|+1)^{(D_{\CV^{(\ell)}}+1)/2} \max_{v^{(\ell)}\in \CV^{(\ell)}} K_{v^{(\ell)}}$.
We then construct the overall network
\[
\phi = \phi_{\CV^{(L)}} \circ \cdots \circ \phi_{\CV^{(1)}},
\]
whose width is $W =\max_\ell\left\{W_{\CV^{(\ell)}}\right\}$, depth is $D=\sum_{\ell=1}^L D_{\CV^{(\ell)}}$, and norm bound satisfies $K\leq2^{D/2}\prod_{\ell} (K_{\CV^{(\ell)}}+2)$ by Proposition~\ref{lem:elementary}.
That is, the resulting network satisfies
\[
\phi \in \CN\CN\left(  \max_\ell\left\{W_{\CV^{(\ell)}}\right\}, \sum_{\ell=1}^L D_{\CV^{(\ell)}},2^{D/2} \prod_{\ell} (K_{\CV^{(\ell)}}+2) \right).
\]

\textbf{(Approximation rate).}
Denote $G_\ell = g_{\CV^{(\ell)}} \circ g_{\CV^{(\ell-1)}} \circ \cdots g_{\CV^{(0)}} $ and $\Phi_\ell = \phi_{\CV_{(\ell)}}\circ \phi_{\CV_{(\ell-1)}} \circ \cdots \circ \phi_{\CV_{(0)}}$.
Then
\begin{align*}
   & |f(\bx)-\phi(\bx)|=|G_L(\bx)-\Phi_L(\bx)|\\
    &= |g_{\CV^{(L)}} \circ \cdots \circ g_{\CV^{(0)}}(\bx)
   - \phi_{\CV^{(L)}} \circ \cdots \circ \phi_{\CV^{(0)}}(\bx)|\\
   & = |g_{\CV^{(L)}} \circ G_{L-1}(\bx)-g_{\CV^{(L)}}\circ \Phi_{L-1}(\bx)|+ |g_{\CV^{(L)}}\circ \Phi_{L-1}(\bx)-\phi_{\CV^{(L)}}\circ \Phi_{L-1}(\bx)|  \\
   &=
   | g_{\CV^{(L)}} \circ G_{L-1}(\bx)-g_{\CV^{(L)}}\circ \Phi_{L-1}(\bx)|+ |g_{\CV^{(L)}}\circ \Phi_{L-1}(\bx)-\phi_{\CV^{(L)}}\circ \Phi_{L-1}(\bx)|
   \\
   &\leq  \left|G_{L-1}(\bx)-\Phi_{L-1} (\bx)\right|_\infty^{\alpha_{v^{(L)}} \wedge 1}+ \|g_{\CV^{(L)}}-\phi_{\CV_{(L)}}\|_{L_\infty([0,1]^{\mathrm{pa}(\CV^{(L)})})}.
\end{align*}
Here, $|G(\bx)|_\infty = \max_{i} (|G(\bx)|)_i$ for any vector-valued function $G(\bx)$. 
In the last step, we use the fact that 
$|f(x_1)-f(x_2)|\le \|f\|_{C^\alpha}\|x_1-x_2\|_\infty^\alpha$ for $0<\alpha\le 1$, and 
$|f(x_1)-f(x_2)|\le \|\nabla f\|_\infty\|x_1-x_2\|_\infty$ for $\alpha>1$, applied to the first term.

For any $\ell = 1, \ldots, L-1$, we derive the following error iteration
\begin{align}\label{equ:error_iter}
&\left|G_{\ell}(\bx)-\Phi_{\ell} (\bx)\right|_\infty \notag\\
&= \max_{v^{(\ell)}\in \mathrm{pa}(v^{(\ell+1)})}
\Bigg\{
\left|g_{v^{(\ell)}}\circ G_{\mathrm{pa}(v^{(\ell)})}(\bx)
      -g_{v^{(\ell)}}\circ \Phi_{\mathrm{pa}(v^{(\ell)})}(\bx)\right| \notag \\
&\qquad\qquad\qquad\qquad
+ \left| g_{v^{(\ell)}}\circ \Phi_{\mathrm{pa}(v^{(\ell)})}(\bx)
      -\phi_{v^{(\ell)}}\circ \Phi_{\mathrm{pa}(v^{(\ell)})}(\bx)\right|
\Bigg\} \notag\\
&\le \max_{v^{(\ell)}\in \mathrm{pa}(v^{(\ell+1)})}
\Bigg\{ 
\left| G_{{\ell-1}}(\bx)-\Phi_{{\ell-1}}(\bx)\right|_\infty^{\alpha_{v^{(\ell)}} \wedge 1} 
+ \left\|g_{v^{(\ell)}}-\phi_{v^{(\ell)}} \right\|_{L_\infty([0,1]^{{d_{\text{in}(v^{(\ell)})}}})}
\Bigg\}.
\end{align}

Recalling ${\mathrm{pa}(v^{(\ell)})}=\{v^{(\ell-1)}:(v^{(\ell-1)},v^{(\ell)})\in \CE^{(\ell)}\}$, 
we write $G_{\mathrm{pa}(v^{(\ell)})}$ and $\Phi_{\mathrm{pa}(v^{(\ell)})}$ instead of $G_{\ell-1}$ and $\Phi_{\ell-1}$ to emphasize the dependency along the path.
To interpret $\max_{v^{(\ell)}\in \mathrm{pa}(v^{(\ell+1)})}$ appearing above, it is necessary to specify which node $v^{(\ell+1)} \in \CV^{(\ell+1)}$ is selected. This choice is determined by the error propagation from the output layer backward. Since there is a single output node at level $L$, the critical path is identified by selecting, at each level, the node that maximizes the propagated error. 
For instance, we define 
\(
v^{(L-1)}_* := \arg\max_{i} \big( |G_{L-1}(\bx)-\Phi_{L-1}(\bx)| \big)_i,
\)
and subsequently select $v^{(L-2)} \in \mathrm{pa}(v^{(L-1)}_*)$, and so on. In this sense, the specific node $v^{(\ell+1)}$ is implicitly determined by the error values established in the preceding levels of the recursion and $\max_{v^{(\ell)}\in \mathrm{pa}(v^{(\ell+1)})}$ is actually $\max_{v^{(\ell)}\in \mathrm{pa}(v^{(\ell+1)})} \cdots \max_{v^{(L-1)}\in \mathrm{pa}(v^{(L)})} (\cdot)$. To simplify notation, let $\pi_{\ell \to L} := v^{(\ell)} \to v^{(\ell+1)}\to \ldots \to v^{(L)}$ 
denote the path selected by the above recursion.
We then write
\[
\max_{v^{(\ell)}\in \mathrm{pa}(v^{(\ell+1)})} \cdots \max_{v^{(L-1)}\in \mathrm{pa}(v^{(L)})} (\cdot)
=
\max_{ \pi_{\ell \to L}} (\cdot).
\]

Denote
\[
A_\ell = |G_\ell(\bx)-\Phi_\ell(\bx)|_\infty, \qquad 
B_\ell = \|g_{v^{(\ell)}}-\phi_{v^{(\ell)}}\|_{L_\infty([0,1]^{d_{\text{in}(v^{(\ell)})}})}, \qquad
\tilde{\alpha}_\ell = \min\{ \alpha_{v^{(\ell)}} , 1\}.
\]
Then, from \eqref{equ:error_iter},
\[
A_L= |G_L(\bx)-\Phi_L(\bx)|, \qquad
A_\ell \le \max_{\pi_{\ell \to L}} \sqrt{d}A_{\ell-1}^{\tilde{\alpha}_{\ell}}+B_\ell,
\qquad A_0=0.
\]

Using $(x+y)^\alpha\le x^\alpha+y^\alpha$ for $\alpha\in(0,1]$ and iterating the recursion, we obtain
\begin{align}\label{equ:path}
 A_L
&\le
\max_{\pi_{1\to L}}
\left\{ 
B_1^{\tilde{\alpha}_2\tilde{\alpha}_3\cdots \tilde{\alpha}_L}
+
d^{(L-2)/2}B_2^{\tilde{\alpha}_3\cdots \tilde{\alpha}_L}
+\cdots+
\sqrt{d}B_{L-1}^{\tilde{\alpha}_L}
+
B_L
\right\}  \\
&\le L  \max_{v^{(\ell)}: \pi_{1 \to L}} 
B_\ell^{\prod_{j=\ell+1}^L \tilde{\alpha}_j}\notag. 
\end{align}
To clarify, the notation $\max_{v^{(\ell)}: \pi_{1 \to L}} $ refers to selecting the node $v^{(\ell)}$ along the path $\pi_{1 \to L}$ that maximizes the quantity $B_\ell^{\prod_{j=\ell+1}^L \tilde{\alpha}_j}$. The exponent $\prod_{j=\ell+1}^L \tilde{\alpha}_j$ represents the cumulative regularity along the subpath $v^{(\ell)} \to \cdots \to v^{(L)}$, which is part of the full path $v^{(1)} \to \cdots \to v^{(L)}$.
In other words, error propagation follows a forward evaluation, while the critical path is identified by backward tracing of the maximal error from the output node.
Therefore,
\begin{equation}\label{equ:final_dif}
|f(x)-\phi(x)|
\leq L   \max_{v^{(\ell)}: \pi_{1 \to L}} \|g_{v^{(\ell)}}-\phi_{v^{(\ell)}}\|_{L_\infty([0,1]^{{d_{\text{in}(v^{(\ell)})}}})}^{\prod_{j=\ell+1}^L \tilde{\alpha}_j}.
\end{equation}

\textbf{(Parameter tuning).}
By \eqref{equ:block_depth} and \eqref{equ:apx_vl}, taking 
\(
N_{v^{(\ell)}}\asymp  \left\lceil k_{v^{(\ell)}}^{2/\alpha_{v^{(\ell)}}}\right\rceil 
\)
gives the width for each node $W_{v^{(\ell)}}\asymp k_{v^{(\ell)}}^{(2 d_{\text{in}(v^{(\ell)})}+\alpha_{v^{(\ell)}})/{\alpha_{v^{(\ell)}}}}$ and the corresponding norm constraint $K_{v^{(\ell)}} \asymp k_{v^{(\ell)}}^{\left(d_{\text{in}({v^{(\ell)}})}(D_{v^{(\ell)}}+1)+2\right)/{\alpha_{v^{(\ell)}}}}$.
Then, the parameters for the overall network $\phi$ become
\[
W=\max_\ell \left\{\sum_{v^{(\ell)}}W_{v^{(\ell)}}\right\}\asymp \max_{\ell,v^{(\ell)}} k_{v^{(\ell)}}^{\frac{2 d_{\text{in}(v^{(\ell)})}+\alpha_{v^{(\ell)}}}{\alpha_{v^{(\ell)}}}}\asymp\max_{\ell,v^{(\ell)}} K_{v^{(\ell)}}^{\frac{2d_{\text{in}(v^{(\ell)})}+\alpha_{v^{(\ell)}}}{2+d_{\text{in}(v^{(\ell)})}\left(D_{v^{(\ell)}} +1\right)}}
\]
\[
D=\sum_{\ell=1}^L \max_{v^{(\ell)}\in \CV^{(\ell)}}D_{v^{(\ell)}}=2L\max_{\ell, v^{(\ell)}}\left\lceil \log_2({d_{\text{in}(v^{(\ell)})}}+{r_{v^{(\ell)}}}) \right\rceil+2L
\]
\[
K = 2^{D/2}\prod_{\ell=1}^L \left( \left(|\CV^{(\ell)}|+1\right)^{\frac{D_{\CV^{(\ell)}}+1}{2}}\max_{v^{(\ell)}\in\CV^{(\ell)}}K_{v^{(\ell)}} +2\right) \asymp \left(\max_{\ell,v^{(\ell)}}K_{v^{(\ell)}}\right)^L.
\]

Therefore,  for any $W \gtrsim \max_{\ell,v^{(\ell)}} K^{\frac{2 d_{\text{in}(v^{(\ell)})}+\alpha_{v^{(\ell)}}}{d_{\text{in}({v^{(\ell)}})}\left(D_{v^{(\ell)}}+1\right)L+2L}}$,
combining \eqref{equ:apx_vl} and \eqref{equ:final_dif}, we obtain
\begin{align*}
\|f - \phi\|_{\infty}
& \le L  
\max_{v^{(\ell)}: \pi_{1 \to L}} 
K_{v^{(\ell)}}^{-\tfrac{2\alpha_{v^{(\ell)}}}{2+d_{\text{in}(v^{(\ell)})}\left(D_{v^{(\ell)}}+1\right)}
      \prod_{m=\ell+1}^{L} \tilde\alpha_{v^{(m)}}} \\
& \le L   
\max_{v^{(\ell)}:\pi_{1 \to L}}
K^{-\frac{2\alpha_{v^{(\ell)}}^*}{2L+(D+L)d_{\text{in}({v^{(\ell)}})}}} 
\end{align*}
where 
$\alpha_{v^{(\ell)}}^*
:= 
\alpha_{v^{(\ell)}} 
\prod_{m=\ell+1}^{L} 
   \left\{\alpha_{v^{(m)}} , 1\right\}.$
This completes the proof.   
\end{proof}




	
	


\section{Proof of the statistical analysis of Frobenius norm-constrained neural networks}\label{sec:pf_stats}
In this section, we combine the approximation error established in Theorem~\ref{thm:apx_cmp} with estimation error bounds to derive the excess risk.

As discussed in Section~\ref{sec:statistical}, the empirical Rademacher complexity is a standard tool for controlling the estimation error.
Let $\mathcal{G}$ be a class of real-valued functions defined on $\mathcal{X}$. The {empirical Rademacher complexity} of $\mathcal{G}$ with respect to a sample \( S = \{x_1, \dots, x_n\} \) is defined as
\[
\hat{\mathfrak{R}}_n(\mathcal{G}) := \mathbb{E}_{\sigma} \left[ \sup_{g \in \mathcal{G}} \frac{1}{n} \sum_{i=1}^n \sigma_i g(x_i) \right],
\]
where \( \sigma_1, \dots, \sigma_n \) are independent Rademacher variables, i.e., random variables uniformly distributed over \( \{-1, +1\} \).
Prior work \citep{neyshabur2015norm,bartlett2017spectrally,golowich2020size} provide Rademacher complexity bounds for function classes generated by neural networks of the form
\begin{equation}\label{equ:nn_absorbed}
   \phi(x)= \widetilde{A}_D \sigma(\widetilde{A}_{D-1} \sigma(\cdots \sigma(\widetilde{A}_0 \tilde{x}))). 
\end{equation}
Here, the bias is absorbed into the augmented matrices $\widetilde{A}_\ell$, where
\[
\widetilde{x} =\begin{pmatrix}
x  \\
1 
\end{pmatrix}, \quad \widetilde{A}_D = (A_D,  0),
\quad
\widetilde{A}_\ell =
\begin{pmatrix}
A_\ell & b_\ell \\
\mathbf{0} & 1
\end{pmatrix},
\quad \ell = 0,\ldots,D-1,
\]
with $\widetilde{A}_D \in \R^{N_{D+1} \times (N_D+1)}$ and $\widetilde{A}_\ell \in \R^{(N_{\ell+1}+1)\times(N_{\ell}+1)}$.
Analogously to the function class $\CN\CN(W,D)$ including \eqref{equ:relu_network}, we denote this function class by $\widetilde{\CN\CN}(W,D)$. The corresponding Frobenius norm constraint is defined as
\begin{equation}\label{eq: kappa equiv}
 \tilde\kappa( \theta) = \prod_{\ell=0}^D\|\tilde{A}_\ell\|_F = \|A_D\|_{F} \prod_{\ell=0}^{D-1} 
\big( \|A_\ell\|_{F}^{2} + \|b_\ell\|_{2}^{2} +1\big)^{1/2}=\kappa(\theta).   
\end{equation}

Analogously to $\CN\CN(W,D,K)$ used in the approximation analysis, we define the constrained class
\[
\widetilde{\CN\CN}(W,D,K)=\{f_\theta \in \widetilde{\CN\CN}(W,D): \tilde\kappa(\theta) \le K \}.
\]

Before presenting the proof of Theorem~\ref{thm:excess_rate}, the following lemma summarizes the relation between the two neural network classes.
\begin{lemma}\label{lem:eqi_nn}
   For any $W, D,$ and $K \ge 0$, we have
\[
\widetilde{\CN\CN}(W, D, K) \equiv 
\CN\CN(W, D, K).
\] 
Moreover, $\tilde k(\theta)=k(\theta)$ for any choice of parameters $\theta$.
\end{lemma}

The above lemma shows that the two neural network classes can essentially be regarded as equivalent, with a slight abuse of notation in identifying $x$ with the augmented variable $\tilde{x}$. This allows us to leverage results from prior work in our analysis. For completeness, we provide the proof below.

\begin{proof}[Proof of Lemma~\ref{lem:eqi_nn}]
We prove the two inclusions separately.

First, consider any function $\tilde f_{ \theta} \in \widetilde{\CN\CN}(W,D,K)$. 
Fixed $\tilde x$, expanding the
augmented matrix multiplication in $\tilde f_\theta(\tilde x)$ \eqref{equ:nn_absorbed}, we recover the
standard biased form \eqref{equ:relu_network} of $f_\theta(x)$. Moreover, by the definition of
the augmented norm, the corresponding parameters satisfy
$\kappa(\theta)=\widetilde\kappa( \tilde \theta)$. Hence the function $\tilde f_\theta$  belongs to
$\CN\CN(W,D,K)$.

Conversely, let $f_\theta\in \CN\CN(W,D,K)$ be represented in the standard
form \eqref{equ:relu_network}. Define the augmented input
$\widetilde x=(x^\top,1)^\top$ and the augmented matrices
\[
\widetilde A_D=(A_D,0),
\qquad
\widetilde A_\ell
=
\begin{pmatrix}
A_\ell & b_\ell\\
\mathbf 0 & 1
\end{pmatrix},
\qquad
\ell=0,\ldots,D-1.
\]
Then
\[
f_\theta(x)
=
\widetilde A_D
\sigma\bigl(
\widetilde A_{D-1}
\sigma(\cdots \sigma(\widetilde A_0\widetilde x))
\bigr).
\]
Therefore $f_\theta\in \widetilde{\CN\CN}(W,D,K)$.
 Furthermore,
\[
\widetilde\kappa(\theta)
=
\|\widetilde A_D\|_F
\prod_{\ell=0}^{D-1}\|\widetilde A_\ell\|_F
=
\|A_D\|_F
\prod_{\ell=0}^{D-1}
\bigl(\|A_\ell\|_F^2+\|b_\ell\|_2^2+1\bigr)^{1/2}
=
\kappa(\theta).
\]
This completes the proof.
\end{proof}


The proof of Theorem~\ref{thm:excess_rate} is now given.

\begin{proof}[Proof of \thmref{thm:excess_rate}] 
\textbf{(Network construction)} To obtain bounded outputs, we apply the truncation operator 
$\chi_1(x) = \max\{-1, \min\{x, 1\}\} \in [-1,1]$ to functions in $\CN\CN(W,D,K)$. 
This operator admits a ReLU realization of the form
\[
\chi_1(x) 
= \sigma(x) - \sigma(-x) 
- \sigma(x - 1)
+ \sigma(-x - 1).
\]


\textbf{(Error estimation)}
With \(
\tilde{f} := \mathop{\arg\min}_{f \in \CN\CN_1(W, D, K)} \CR(f)
\), the excess risk can be decomposed as
 \begin{align}\label{pf_errdcmp}
&\CR(\hat{f}) - \CR(f_*) \notag\\
&= \CR(\hat{f}) - \hat{\CR}(\hat{f}) + \hat{\CR}(\hat{f}) - \hat{\CR}(\tilde{f}) + \hat{\CR}(\tilde{f}) - \CR(\tilde{f}) + \CR(\tilde{f}) - \CR(f_*) \notag\\
&\le \CR(\hat{f}) - \hat{\CR}(\hat{f}) + \hat{\CR}(\tilde{f}) - \CR(\tilde{f}) + \CR(\tilde{f}) - \CR(f_*) \notag \\
& \leq \sup_{f\in \CN\CN_1({W,D,K})} \{\CR({f}) - \hat{\CR}({f}) \}
+
\sup_{f\in \CN\CN_1({W,D,K})}  \{\hat{\CR}(f) - \CR(f) \}+ \CR(\tilde{f}) - \CR(f_*).
\end{align} 

For the approximation error $\CR(\tilde{f}) - \CR(f_*)$, let $\phi \in \CN\CN(W,D,K)$ and define $f=\chi_1 \circ \phi \in \CN\CN_1(W,D,K)$. Then, for all $x \in X$,
\(
|f(x)-f_*(x)| \leq |\phi(x)-f_*(x)|,
\)
since $f_* \in [-1,1]$ and $\chi_1$ is $1$-Lipschitz. Thus,
it follows from \thmref{thm:apx_cmp} that
\begin{align}\label{equ:app_tr}
  &\CR(\tilde f)-\CR(f_*)
=\inf_{f \in \CN\CN_1(W,D,K)}\|f-f_*\|_2^2 \notag \\ &
\le \inf_{\phi \in \CN\CN( W, D, K)}\|\phi-f_*\|_2^2 
\le 
\max_{v^{(\ell)}:\pi_{1 \to L}}
K^{-\frac{4\alpha_{v^{(\ell)}}^*}{2L+ (D+L) d_{\text{in}(v^{(\ell)})}}}.  
\end{align}

We proceed by bounding the first estimation error term in \eqref{pf_errdcmp}. The second term can be bounded similarly.
Let $\mathcal{G}$ be a class of real-valued functions defined on ${X} \times \R$, such as
\[
\mathcal{G} := \left\{g_f\ \middle| \ g_f (x, y) = \left( f(x) - y \right)^2, f \in \CN\CN_1(W, D, K) \right\}.
\] 
For any $f \in \CN\CN_1(W,D,K)$, we have
$||g_f||_\infty  \leq 4$ since $|y|\leq 1$. 
Moreover, it is straightforward to observe that
\[
\sup_{f \in \CN\CN_1(W, D, K)}
\CR({f}) -  \hat{\CR}({f})  =
\sup_{g_f \in \mathcal G}
\left( \E_\rho[g_f(x,y)] - \frac{1}{n}\sum_{i=1}^n g_f(x_i,y_i) \right).
\]
By the standard one-sided Rademacher complexity bound and McDiarmid inequality for uniformly bounded
function classes, see for example \citep[Theorem~3.3]{mohri2018foundations},
with probability at least \(1-\delta/2\),
\[
\sup_{g_f\in\mathcal G}
\left\{
\E_\rho[g_f]
-
\frac1n\sum_{i=1}^n g_f(x_i,y_i)
\right\}
\le
2\hat{\mathfrak R}_n(\mathcal G)
+
12\sqrt{\frac{\log(4/\delta)}{2n}}.
\]
To bound the complexity of $\CG$, we apply the contraction principle of Ledoux--Talagrand 
\citep{ledoux1991probability,bartlett2002rademacher}, which states that the Rademacher complexity is stable under Lipschitz transformations. 
Observe that
\(
g_f(x,y) 
\)
is a Lipschitz function of $f(x)$ with Lipschitz constant bounded by $4$. 
Therefore, as $\chi_1$ is Lipschitz as well, we obtain
\[
\hat{\FR}_n(\CG) 
\le
4\,\hat{\FR}_n(\CN\CN_1(W,D,K))
\leq 
4
\,\hat{\FR}_n(\CN\CN(W,D,K))
\]

From Lemma~\ref{lem:eqi_nn} and \citep[Theorem 3.1]{golowich2020size}, the Rademacher complexity of norm-constrained neural networks can be bounded by 
\[\hat{\FR}_n(\CN\CN(W,D,K)) \leq \hat{\FR}_n(\widetilde{\CN\CN}(W+1,D,K)) \le \frac{\left(\sqrt{2\log(2)D}+1\right)K}{\sqrt n}.\]

Therefore, with probability at least $1-\delta/2$, the first estimation error 
\[\sup_{f \in \CN\CN_1(W, D, K)}
\CR({f}) -  \hat{\CR}({f})  \le 8\frac{(\sqrt{2\log (2)D}+1)\;K}{\sqrt{n}}+
12\sqrt{\frac{\log(4/\delta)}{2n}}.\]



Combining this with the error decomposition \eqref{pf_errdcmp} and the above approximation bound \eqref{equ:app_tr}, with probability at least $1-\delta$, there holds 
\begin{align*}
\CR(\hat{f}) - \CR(f_*)
&\lesssim \frac{K \sqrt{D}}{\sqrt{n}} 
+ \sqrt{\frac{\log(4/\delta)}{2n}}
+ L \max_{v^{(\ell)}:\pi_{1 \to L}}
K^{-\frac{4\alpha_{v^{(\ell)}}^*}{2L+(D+L) d_{\text{in}(v^{(\ell)})}}}.
\end{align*}

By taking 
\[
K
\asymp\max_{v^{(\ell)}:\pi_{1 \to L}}
\left(n\right)^{
\frac12
\frac{2L+(D+L)d_{\text{in}({v^{(\ell)}})}}{
2L+(D+L)d_{\text{in}({v^{(\ell)}})}+4\alpha_{v^{(\ell)}}^*}},
\]
 the following holds with probability at least $1-\delta$,
\[
\CR(\hat f)-\CR(f_*)
\lesssim \sqrt{\log \frac{4}{\delta}}
\max_{v^{(\ell)}:\pi_{1 \to L}}n^{-
\frac{2\alpha_{v^{(\ell)}}^*}{
2L+(D+L)d_{\text{in}({v^{(\ell)}})}+4\alpha_{v^{(\ell)}}^*
}}.
\]
This completes the proof. 

%
\end{proof}

\newpage
\input{checklist}

\end{document}

%% file: macros.tex
\def \bn {\mathbf{n}}
\def \bs {\mathbf{s}}
\def \bx {\mathbf{x}}

\newcommand{\be}{\begin{equation}}
\newcommand{\ee}{\end{equation}}
\newcommand{\br}{\begin{rem}}
\newcommand{\er}{\end{rem}}
\newcommand{\bq}{\begin{qu}}
\newcommand{\eq}{\end{qu}}
\newcommand{\en}{\end{enumerate}}
\newcommand{\bi}{\begin{itemize}}
\newcommand{\ei}{\end{itemize}}
\newcommand{\beas}{\begin{eqnarray*}}
\newcommand{\eeas}{\end{eqnarray*}}
\newcommand{\bea}{\begin{eqnarray}}
\newcommand{\eea}{\end{eqnarray}}

\newcommand{\E}{\mathbb{E}}

\newcommand{\R}{\mathbb R}

\def\CC{\mathcal C} 
\def\CE{\mathcal E} 

\def\CG{\mathcal G} 

\def\CN{\mathcal N} 
\def\CR{\mathcal R} 
\def\CV{\mathcal{V}}

\def\EE{\mathbb{E}} 
\def\NN{\mathbb{N}} 

\def \FR {\mathfrak{R}}

\def \md {\mathrm{d}} 

\newtheorem{theorem}{Theorem}
\newtheorem{lemma}{Lemma}
\newtheorem{ass}{Assumption}
\newtheorem{proposition}{Proposition}
\newtheorem{rem}{Remark}

\newcommand{\thmref}[1]{Theorem~\ref{#1}}



%% file: checklist.tex
\section*{NeurIPS Paper Checklist}

\begin{enumerate}

\item {\bf Claims}
    \item[] Question: Do the main claims made in the abstract and introduction accurately reflect the paper's contributions and scope?
    \item[] Answer: \answerYes{} 
    \item[] Justification: The main claims are clearly stated in the abstract and introduction, and are formalized in Theorems~\ref{thm:apx_cmp} and \ref{thm:excess_rate}, including assumptions and scope.

\item {\bf Limitations}
    \item[] Question: Does the paper discuss the limitations of the work performed by the authors?
    \item[] Answer: \answerYes{} 
    \item[] Justification: The paper discusses key assumptions (e.g., bounded outputs, smoothness, DAG structure, norm constraints) and their implications; see Assumptions~\ref{ass:smooth}, \ref{ass:boundeddata}, and discussion after theorems.

\item {\bf Theory assumptions and proofs}
    \item[] Question: For each theoretical result, does the paper provide the full set of assumptions and a complete (and correct) proof?
    \item[] Answer: \answerYes{} 
    \item[] Justification: All assumptions are explicitly stated, and complete proofs are provided in the appendix with proper cross-referencing.

    \item {\bf Experimental result reproducibility}
    \item[] Question: Does the paper fully disclose all the information needed to reproduce the main experimental results of the paper to the extent that it affects the main claims and/or conclusions of the paper (regardless of whether the code and data are provided or not)?
    \item[] Answer:  \answerNA{}
    \item[] Justification: This paper is theoretical and does not include empirical experiments.

\item {\bf Open access to data and code}
    \item[] Question: Does the paper provide open access to the data and code, with sufficient instructions to faithfully reproduce the main experimental results, as described in supplemental material?
    \item[] Answer: \answerNA{}
    \item[] Justification: This work does not involve datasets, code, or empirical evaluation.

\item {\bf Experimental setting/details}
    \item[] Question: Does the paper specify all the training and test details (e.g., data splits, hyperparameters, how they were chosen, type of optimizer) necessary to understand the results?
    \item[] Answer: \answerNA{} 
    \item[] Justification: No experimental setup is involved.

\item {\bf Experiment statistical significance}
    \item[] Question: Does the paper report error bars suitably and correctly defined or other appropriate information about the statistical significance of the experiments?
    \item[] Answer: \answerNA{} 
    \item[] Justification: The paper does not include experiments.

\item {\bf Experiments compute resources}
    \item[] Question: For each experiment, does the paper provide sufficient information on the computer resources (type of compute workers, memory, time of execution) needed to reproduce the experiments?
    \item[] Answer: \answerNA{} 
    \item[] Justification: No computational experiments are conducted.
    
\item {\bf Code of ethics}
    \item[] Question: Does the research conducted in the paper conform, in every respect, with the NeurIPS Code of Ethics \url{https://neurips.cc/public/EthicsGuidelines}?
    \item[] Answer: \answerYes{} 
    \item[] Justification: The research complies with the NeurIPS Code of Ethics and does not involve sensitive data, human subjects, or harmful applications.

\item {\bf Broader impacts}
    \item[] Question: Does the paper discuss both potential positive societal impacts and negative societal impacts of the work performed?
    \item[] Answer: \answerNA{} 
    \item[] Justification: This work is foundational and theoretical, without a direct application pathway leading to specific societal impacts.
    
\item {\bf Safeguards}
    \item[] Question: Does the paper describe safeguards that have been put in place for responsible release of data or models that have a high risk for misuse (e.g., pre-trained language models, image generators, or scraped datasets)?
    \item[] Answer: \answerNA{} 
    \item[] Justification: The paper does not release data, models, or systems that pose risks of misuse.

\item {\bf Licenses for existing assets}
    \item[] Question: Are the creators or original owners of assets (e.g., code, data, models), used in the paper, properly credited and are the license and terms of use explicitly mentioned and properly respected?
    \item[] Answer: \answerNA{} 
    \item[] Justification: The paper does not use existing assets

\item {\bf New assets}
    \item[] Question: Are new assets introduced in the paper well documented and is the documentation provided alongside the assets?
    \item[] Answer: \answerNA{} 
    \item[] Justification: The paper does not introduce new datasets, code, or models.

\item {\bf Crowdsourcing and research with human subjects}
    \item[] Question: For crowdsourcing experiments and research with human subjects, does the paper include the full text of instructions given to participants and screenshots, if applicable, as well as details about compensation (if any)? 
    \item[] Answer: \answerNA{} 
    \item[] Justification: No human subjects or crowdsourcing are involved.

\item {\bf Institutional review board (IRB) approvals or equivalent for research with human subjects}
    \item[] Question: Does the paper describe potential risks incurred by study participants, whether such risks were disclosed to the subjects, and whether Institutional Review Board (IRB) approvals (or an equivalent approval/review based on the requirements of your country or institution) were obtained?
    \item[] Answer: \answerNA{} 
    \item[] Justification: The paper does not involve human subjects.

\item {\bf Declaration of LLM usage}
    \item[] Question: Does the paper describe the usage of LLMs if it is an important, original, or non-standard component of the core methods in this research? Note that if the LLM is used only for writing, editing, or formatting purposes and does \emph{not} impact the core methodology, scientific rigor, or originality of the research, declaration is not required.
    \item[] Answer: \answerNA{} 
    \item[] Justification: The core method development in this research does not involve LLMs as any important, original, or non-standard components.

\end{enumerate}